\documentclass{article}

\usepackage{graphicx}       
\usepackage[utf8]{inputenc}
\usepackage[T1]{fontenc}
\usepackage{hyperref}       
\hypersetup{
	pdftitle={Linear--Nonlinear Fusion Neural Operator for Partial Differential Equations},
	pdfauthor={Heng Wu, Junjie Wang, Benzhuo Lu},
	pdfkeywords={Neural operators, PDE, Machine Learning},
	pdfsubject={Mathematics, Machine Learning}
}
\usepackage{xcolor}         
\usepackage{amsmath,amssymb,amsfonts}       
\usepackage{booktabs}       
\usepackage{multirow}
\usepackage{pdfpages}
\usepackage{authblk}
\usepackage{caption}
\usepackage{geometry} 
\renewcommand\thetable{\arabic{table}}

\pdfoutput=1

\usepackage{amsthm}
\newtheorem{theorem}{Theorem}[section]

\title{Linear--Nonlinear Fusion Neural Operator for Partial Differential Equations}

\author[1,2]{Heng Wu}
\author[1,2]{Junjie Wang}
\author[1,2]{Benzhuo Lu\thanks{Corresponding author. E-mail: \texttt{bzlu@lsec.cc.ac.cn}}}

\affil[1]{SKLMS, ICMSEC, NCMIS, Academy of Mathematics and Systems Science, Chinese Academy of Sciences, Beijing 100190, China}
\affil[2]{School of Mathematical Sciences, University of Chinese Academy of Sciences, Beijing 100049, China}

\date{}

\begin{document}
	
	\maketitle
	
\begin{abstract}
	Neural operator learning directly constructs the mapping relationship from the equation parameter space to the solution space, enabling efficient direct inference in practical applications without the need for repeated solution of partial differential equations (PDEs)---an advantage that is difficult to achieve with traditional numerical methods.
	In this work, we find that explicitly decoupling linear and nonlinear effects within such operator mappings leads to improved learning efficiency.
	This yields a novel network structure, namely the Linear--Nonlinear Fusion Neural Operator (LNF-NO), which models operator mappings via the multiplicative fusion of a linear component and a nonlinear component, thus achieving a lightweight and interpretable representation.
	This linear--nonlinear decoupling enables efficient capture of complex solution features at the operator level while maintaining stability and generality.
	LNF-NO naturally supports multiple functional inputs and is applicable to both regular grids and irregular geometries.
	Across a diverse suite of PDE operator-learning benchmarks, including nonlinear Poisson--Boltzmann equations and multi-physics coupled systems, LNF-NO is typically substantially faster to train than several representative neural operator baselines, while achieving comparable or improved accuracy across most tested cases.
	On the tested 3D Poisson--Boltzmann case, LNF-NO achieves strong accuracy while requiring substantially less training time than the three-dimensional Fourier Neural Operator and Transolver baselines.
\end{abstract}

\section{Introduction}
\label{sec:intro}

Neural operator learning has emerged as a central paradigm in scientific machine learning for approximating solution operators of partial differential equations (PDEs).
By learning mappings between function spaces, neural operators aim to amortize the cost of repeatedly solving PDEs under varying boundary conditions, source terms, or coefficients.
Such many-query settings arise ubiquitously in physical, chemical, and biological modeling, where repeated numerical solves constitute a major computational bottleneck \cite{Quarteroni2015, Benner2015, Willcox2002}.
Neural operators provide fast surrogate models that deliver approximate PDE solutions with accuracy often sufficient for downstream analysis and decision-making, while substantially reducing computational cost compared with repeated high-fidelity simulations \cite{Lu_2021, Kovachki2023JMLR}.

Among existing approaches, the Deep Operator Network (DeepONet) \cite{Lu_2021} and the Fourier Neural Operator (FNO) \cite{li2021fourier} represent the most widely used and representative neural operator models.
DeepONet employs a branch--trunk decomposition to learn operator mappings from input functions to solution fields, while FNO leverages global convolution in the Fourier domain to efficiently capture long-range dependencies on regular grids.
Beyond these baselines, a growing family of neural operator architectures has been proposed, including geometry-informed operators \cite{Li2023GINO}, the U-shaped Neural Operator (UNO) \cite{Rahman2023UNO}, and transformer-inspired or attention-based formulations such as Transolver \cite{Wu2024Transolver} and PDEformer \cite{Ye2024PDEformer}.

\begin{figure*}[t]
	\centering
	\includegraphics[width=0.9\linewidth]{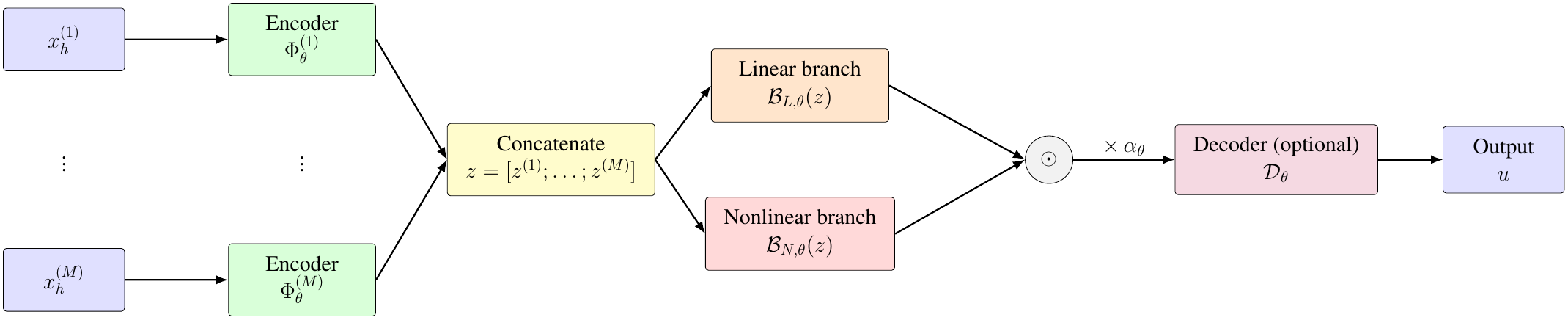}
	\caption{
		Architecture of the proposed Linear--Nonlinear Fusion Neural Operator (LNF-NO).
		Each input component (typically a discretized function, e.g., boundary traces or source fields) is encoded separately and concatenated into a latent representation.
		An operator core then fuses a linear branch and a nonlinear branch via element-wise multiplication ($\odot$), producing a raw prediction that can be optionally refined by a lightweight decoder.
		Arrows indicate the data flow.
	}
	\label{fig:lnfno_arch}
\end{figure*}

Despite this progress, several challenges remain in practical operator-learning settings.
First, nonlinear interactions in PDE operators can be difficult to represent or optimize using purely additive architectural compositions.
Second, many real-world problems involve \emph{multiple functional inputs}, such as boundary traces, source terms, or coefficient fields, as well as coupled multi-field outputs.
Third, training efficiency and optimization stability become increasingly important as problem complexity grows, particularly for three-dimensional PDEs and large-scale benchmarks.

Non-additive interactions have proven effective across a wide range of deep learning architectures.
Feature-wise modulation mechanisms, such as Feature-wise Linear Modulation (FiLM), apply learned, input-dependent scaling to intermediate representations and can enhance expressivity with minimal computational overhead \cite{perez2018film}.
More broadly, attention mechanisms \cite{vaswani2017attention} and mixture-of-experts models \cite{shazeer2017outrageously} illustrate that structured, conditional interactions between multiple components can improve representational capacity and training efficiency.
These developments motivate exploring multiplicative fusion as a general design principle for neural operator architectures.

Our architectural design is further motivated by numerical insights into semi-linear elliptic PDEs and coupled systems.
For instance, in biomolecular electrostatics and ion transport, solutions to nonlinear equations such as Poisson--Boltzmann or Poisson--Nernst--Planck are often viewed as being predominantly governed by their corresponding linear approximations, modulated by nonlinear correction terms \cite{lu2008recent, zhang2020model}.
Inspired by this, we hypothesize that explicitly disentangling these linear and nonlinear effects within the neural architecture serves as a useful physics-aware inductive bias.

In this work, we propose the Linear--Nonlinear Fusion Neural Operator (LNF-NO), which models operator behavior through the multiplicative interaction of a linear component and a nonlinear component.
LNF-NO treats these components symmetrically: the two branches interact directly through element-wise fusion, rather than being combined through hierarchical or purely additive structures.
This design provides a lightweight and interpretable mechanism for capturing nonlinear effects while maintaining architectural simplicity and stable optimization behavior.
LNF-NO naturally supports multiple functional inputs by encoding each input function independently and fusing their latent representations within the operator core.
The architecture can be paired with a lightweight grid-based decoder for standard benchmark problems, or used in a decoder-free setting for irregular geometries where solutions are predicted at sampled points.

We evaluate LNF-NO on a diverse suite of PDE operator-learning benchmarks, including nonlinear Poisson--Boltzmann equations, coupled multi-field systems, and domains with irregular geometries.
Across these tasks, LNF-NO consistently improves training efficiency relative to representative neural operator baselines, while achieving comparable or improved accuracy across most tested cases.
In particular, we present a three-dimensional Poisson--Boltzmann case study in which LNF-NO achieves strong accuracy while requiring substantially less training time than the three-dimensional Fourier Neural Operator and Transolver baselines.

The main contributions of this work are summarized as follows:
\begin{itemize}
	\item We introduce a linear--nonlinear multiplicative fusion mechanism that explicitly decouples complementary effects. This design acts as an efficiency-oriented inductive bias, yielding substantially improved training efficiency on a broad range of benchmarks while achieving comparable or improved accuracy across most tested cases.
	
	\item Our simple and flexible architecture naturally accommodates multiple functional inputs and predicts multiple coupled output fields within a unified framework, seamlessly generalizing across both regular grids and irregular geometries while preserving the same linear--nonlinear fusion core.
	
	\item We demonstrate the scalability and robustness of LNF-NO on a diverse suite of benchmarks, including three-dimensional systems and strongly nonlinear equations (e.g., Poisson--Boltzmann), where the proposed architecture maintains stable optimization and strong empirical performance.
\end{itemize}
	
\section{Methodology}
\label{sec:methods}

\subsection{Operator learning with input/output discretization}

Let $\mathcal{G}:\mathcal{X}\to\mathcal{Y}$ denote the solution operator of a PDE posed on a bounded domain
$\Omega\subset\mathbb{R}^d$, mapping input functions (e.g., boundary traces, coefficient fields, or source terms)
to solution fields.
In practice, operator learning is carried out after discretizing \emph{both} inputs and outputs.

\textbf{Input discretization.}
We consider \emph{multi-function} inputs with $M$ components
\[
x = \big(x^{(1)},\dots,x^{(M)}\big),\qquad x^{(m)}\in\mathcal{X}^{(m)}.
\]
Each component is discretized by a sampling or projection operator
$\mathcal{I}^{(m)}_h:\mathcal{X}^{(m)}\to\mathbb{R}^{d^{(m)}_{\mathrm{in}}}$,
yielding $x_h^{(m)}:=\mathcal{I}^{(m)}_h(x^{(m)})$.
Stacking all components gives $x_h:=\mathcal{I}_h(x)\in\mathbb{R}^{d_{\mathrm{in}}}$,
where $d_{\mathrm{in}}=\sum_{m=1}^M d^{(m)}_{\mathrm{in}}$.
In general, $\mathcal{I}_h$ is not invertible.

\textbf{Data-generating family and lifting.}
The dataset is generated from a restricted family $\mathcal{A}\subset\mathcal{X}$.
Accordingly, we only require the discrete learning target to be well defined on the set of discrete inputs observed in data.
Let $K_h:=\mathcal{I}_h(\mathcal{A})\subset\mathbb{R}^{d_{\mathrm{in}}}$.
We fix a (possibly non-unique) selection map
$\mathcal{L}_h:K_h\to\mathcal{A}$ such that
$\mathcal{I}_h(\mathcal{L}_h(x_h))=x_h$ for all $x_h\in K_h$.
The choice of $\mathcal{L}_h$ reflects the data-generation procedure and does not require uniqueness.

\textbf{Output discretization.}
Let $\mathcal{I}^{\mathrm{out}}_h:\mathcal{Y}\to\mathbb{R}^{d_{\mathrm{out}}}$ denote an output sampling or projection map,
such as evaluation on a fixed grid or a fixed point set for irregular domains.
For $C_{\mathrm{out}}$ output fields, typically $d_{\mathrm{out}}=C_{\mathrm{out}}N_{\mathrm{out}}$.

\textbf{Discrete target operator.}
We define the discrete operator learning target
\begin{equation}
	F_h^\star := \mathcal{I}^{\mathrm{out}}_h \circ \mathcal{G} \circ \mathcal{L}_h,
	\qquad
	F_h^\star:K_h\to\mathbb{R}^{d_{\mathrm{out}}}.
	\label{eq:discrete_target}
\end{equation}
All models are trained and evaluated in the discrete output space $\mathbb{R}^{d_{\mathrm{out}}}$.

\subsection{LNF-NO architecture (multi-input, multi-output)}

Fig.~\ref{fig:lnfno_arch} overviews the LNF-NO architecture, generally formulated to handle multiple functional inputs (e.g., boundary conditions and source terms) and predict multiple coupled output fields (multi-output).
Given $x_h=\big(x_h^{(1)},\dots,x_h^{(M)}\big)$, each input component is encoded independently:
\begin{align*}
	z^{(m)} &= \Phi^{(m)}_\theta\!\left(x_h^{(m)}\right)\in\mathbb{R}^{d_m}, \qquad m=1,\dots,M, \\
	z &= \big[z^{(1)};\ldots;z^{(M)}\big]\in\mathbb{R}^{d_z}.
\end{align*}

To balance expressivity and optimization efficiency, the operator core consists of a linear branch and a nonlinear branch fused multiplicatively:
\begin{equation}
	u_{\mathrm{raw}} = \alpha_\theta \cdot \left( \mathcal{B}_{L,\theta}(z) \odot \mathcal{B}_{N,\theta}(z) \right)\in\mathbb{R}^{d_{\mathrm{out}}},
	\label{eq:multiplicative_core}
\end{equation}
where $\odot$ denotes element-wise multiplication and $\alpha_\theta$ is a learnable scalar (see Appendix~\ref{app:lnfno_arch} for initialization details).
Both branches map the latent representation to the discrete output space,
$\mathcal{B}_{L,\theta},\mathcal{B}_{N,\theta}:\mathbb{R}^{d_z}\to\mathbb{R}^{d_{\mathrm{out}}}$,
with $\mathcal{B}_{L,\theta}$ being affine and $\mathcal{B}_{N,\theta}$ a nonlinear multilayer perceptron (MLP).

\medskip
\noindent\textbf{Intuition.}
The proposed fusion strategy is conceptually grounded in both numerical analysis and deep learning.
Physically, the explicit decoupling of the operator into linear and nonlinear branches mirrors operator decomposition strategies inspired by numerical practices and empirical insights (see, e.g., \cite{lu2008recent} for Poisson--Boltzmann and \cite{zhang2020model} for Poisson--Nernst--Planck, where solutions are often globally close to their linear counterparts).
Structurally, in contrast to additive compositions such as residual connections in ResNets \cite{He2016ResNet} which learn corrections via summation, our multiplicative fusion allows the nonlinear branch to spatially gate or scale this dominant linear baseline in an input-adaptive manner.
This parallels feature-wise modulation mechanisms (e.g., FiLM \cite{perez2018film}) but is explicitly formulated to capture the coupled linear--nonlinear nature of PDE operators.
From an optimization perspective, this factorized structure can provide a favorable inductive bias by allowing the linear branch to capture a dominant baseline response, while the nonlinear branch focuses on input-dependent modulation, which may simplify the learning of complex operator mappings in practice.

Optionally, a lightweight decoder $\mathcal{D}_\theta$ refines local consistency on regular grids:
\[
F_\theta(x_h)=
\begin{cases}
	u_{\mathrm{raw}}, & \text{decoder-free},\\
	\mathcal{D}_\theta(u_{\mathrm{raw}}), & \text{with decoder}.
\end{cases}
\]

\subsection{Approximation Guarantee}

\begin{theorem}[Universal approximation]
	\label{thm:uat}
	Assume $K_h\subset\mathbb{R}^{d_{\mathrm{in}}}$ is compact and
	$F_h^\star\in C(K_h;\mathbb{R}^{d_{\mathrm{out}}})$.
	If LNF-NO employs non-polynomial activations, such as Rectified Linear Unit (ReLU), Gaussian Error Linear Unit (GELU), or $\tanh$,
	in its feedforward components, then for any $\varepsilon>0$
	there exist parameters $\theta$ such that
	\[
	\sup_{x_h\in K_h}\|F_\theta(x_h)-F_h^\star(x_h)\|\le \varepsilon.
	\]
\end{theorem}

\emph{Proof sketch.}
By selecting encoder and decoder parameters so that they reduce to identity mappings
(or fixed linear embeddings) on $K_h$,
and setting $\mathcal{B}_{L,\theta}(z)\equiv\mathbf{1}$ and $\alpha_\theta=1$,
the architecture contains a standard feedforward network
from $\mathbb{R}^{d_{\mathrm{in}}}$ to $\mathbb{R}^{d_{\mathrm{out}}}$ as a special case.
Universal approximation then follows from classical results
for non-polynomial activations~\cite{Cybenko1989,Hornik1991}.
\hfill$\square$

\subsection{Optimization}

We train LNF-NO by supervised learning on pairs $(x_h,y_h)$ with $y_h=F_h^\star(x_h)$.
We use the averaged relative $\ell_2$ error across output fields as the loss:
\begin{equation}
	\mathcal{L}(\theta)
	=
	\frac{1}{B}\sum_{b=1}^B
	\frac{1}{C_{\mathrm{out}}}\sum_{c=1}^{C_{\mathrm{out}}}
	\frac{\|\hat y_{h,b}^{(c)}-y_{h,b}^{(c)}\|_2}{\|y_{h,b}^{(c)}\|_2+\varepsilon}.
	\label{eq:loss_multifield}
\end{equation}
All losses and metrics are computed on the decoded physical scale after inverse normalization.
We optimize using AdamW~\cite{Loshchilov2019AdamW}; task-specific hyperparameters, including learning rates and training schedules, are provided in the appendix.
Following common practice, we exclude bias terms and the scalar fusion scale $\alpha_\theta$ from weight decay, since they act as calibration parameters rather than capacity-controlling weights.
	
\section{Results}
\label{sec:results}

\textbf{Evaluation protocol.}
All models are trained and evaluated under a common evaluation protocol.
For each task, we use identical training and test sets across all methods with a fixed $9{:}1$ split,
and report the mean relative $\ell_2$ error on the test set.
DeepONet is trained for 5000 epochs, while the other neural operator baselines and LNF-NO are trained for 500 epochs.
Training time is measured on the same hardware and includes the full training process.
Task-specific hyperparameters, including learning rates, batch sizes, and architectural settings, are provided in the appendix.

\textbf{Task definitions.}
In all benchmarks, we learn solution operators that map boundary conditions
and/or input fields to solution fields evaluated on a fixed set of points.
We summarize the PDE tasks considered in this work below.

\emph{Laplace.}
We study the 2D Laplace equation
$\Delta u = 0$ in $\Omega=(0,1)^2$ with Dirichlet boundary condition $u|_{\partial\Omega}=g$,
and learn the operator mapping the boundary trace $g$ to the solution field $u$
on a fixed grid (see Appendix~\ref{app:laplace_mad}).

\emph{Burgers.}
We consider the 1D viscous Burgers equation
\begin{equation}
	\partial_t u + u\,\partial_x u = \nu\,\partial_{xx}u
	\quad \text{on } x\in[0,2\pi),~t\in[0,T],
\end{equation}
with periodic boundary conditions.
The operator maps the initial condition $u(\cdot,0)$ to the solution trajectory
$u(\cdot,t)$ on a fixed space--time grid
(see Appendix~\ref{app:burgers_lowmodes}).

\emph{Darcy flow.}
We benchmark the elliptic Darcy equation
$-\nabla\cdot(a(x)\nabla u)=f$ in $\Omega=(0,1)^2$,
where $a(x)$ denotes the permeability field.
Both smooth and piecewise constant permeability regimes are considered
(see Appendix~\ref{app:darcy}).

\emph{Poisson--Boltzmann (PB).}
We evaluate the nonlinear Poisson--Boltzmann operator
\begin{equation}
	-\Delta u + k\,\sinh(u) = f \quad \text{in } \Omega,
	\qquad u|_{\partial\Omega}=g,
	\label{eq:pb}
\end{equation}
where $k>0$ partially controls the strength of the nonlinearity.
The source-free case corresponds to $f=0$, while the source-driven case uses $f\neq 0$
(see, e.g.,~\cite{lu2008recent} for background on numerical PB models
and Appendix~\ref{app:pb} for benchmark details).

\emph{Irregular-domain PB.}
To evaluate geometric generalization, we further consider source-free PB problems
posed on irregular planar domains with multiple boundary components,
discretized using finite element methods and represented in a node-based format
(see Appendix~\ref{app:pb_irregular_fem}).

\emph{Poisson--Nernst--Planck (PNP).}
Finally, we benchmark a multi-field coupled operator corresponding to the Poisson--Nernst--Planck system (see, e.g.,~\cite{zhang2020model} for the classical PNP formulation and Appendix~\ref{app:pnp} for dataset details).
A dimensionless form is taken where physical parameters are set to unity for operator-learning purposes:
\begin{equation}
	-\Delta \phi = c_+ - c_-, \qquad \nabla \cdot (\nabla c_\pm \pm c_\pm \nabla \phi) = 0.
	\label{eq:pnp_simplified}
\end{equation}
The operator maps boundary traces to the electric potential $\phi$ and ion concentration fields $c_\pm$ on a fixed grid.

\subsection{Baseline benchmarks on regular grids}

\subsubsection{Single-input operators}

We first consider standard \emph{single-input} operator-learning benchmarks on regular grids,
including Laplace, Burgers, Darcy flow (smooth and discontinuous coefficients), and the source-free Poisson--Boltzmann equation ($k=1$).
These tasks span linear and nonlinear elliptic and parabolic operators
with varying degrees of complexity.

\begin{table*}[t]
	\centering
	\caption{Baseline benchmarks on regular grids (single-input operators).
		We report the test mean relative $\ell_2$ error and wall-clock training time under the unified protocol.}
	\label{tab:baseline_single}
	\begin{sc}
		\renewcommand{\arraystretch}{1.0}
		\begin{tabular}{l l c c c c}
			\toprule
			Equation & Model & \#Params & Epochs & Training Time (s) & Test rel.\ $\ell_2$ \\
			\midrule
			\multirow{5}{*}{Laplace}
			& DeepONet & 973,313 & 5000 & 1470.08 & 1.16e-02 \\
			& FNO      & 9,133,866 & 500 & 999.22 & 1.66e-02 \\
			& UNO      & 9,207,722 & 500 & 1393.00 & 2.05e-02 \\
			& Transolver & 2,158,538 & 500 & 7322.38 & 1.04e-02\\
			& LNF-NO   & 2,891,956 & 500 & \textbf{257.57} & \textbf{7.43e-03} \\
			\midrule
			\multirow{5}{*}{Burgers}
			& DeepONet & 938,497 & 5000 & 1760.09 & 1.05e-01 \\
			& FNO      & 8,413,953 & 500 & 2316.45 & 1.20e-02 \\
			& UNO      & 8,487,809 & 500 & 3143.00 & \textbf{1.10e-02} \\
			& Transolver & 1,438,625 & 500 & 43598.27 & 2.43e-02\\
			& LNF-NO   & 4,037,346 & 500 & \textbf{911.87} & 3.89e-02 \\
			\midrule
			\multirow{5}{*}{\shortstack{Darcy\\(Smooth Coeff.)}}
			& DeepONet & 5,182,209 & 5000 & 4981.13 & 5.65e-02 \\
			& FNO      & 8,413,953 & 500 & 3889.17 & 2.58e-03 \\
			& UNO      & 9,587,201 & 500 & 6830.23 & 2.39e-03 \\
			& Transolver & 1,438,625 & 500 & 302922.64 & 3.82e-03\\
			& LNF-NO   & 5,922,948 & 500 & \textbf{603.17} & \textbf{1.79e-03} \\
			\midrule
			\multirow{5}{*}{\shortstack{Darcy\\(Piecewise Coeff.)}}
			& DeepONet & 15,790,849 & 5000 & 36945.35 & 4.87e-02 \\
			& FNO      & 8,413,953 & 500 & 11127.40 & 1.24e-02 \\
			& UNO      & 9,587,201 & 500 & 119219.69 & \textbf{8.65e-03} \\
			& Transolver & 1,438,625 & 500 & 534539.25 & 8.78e-03\\
			& LNF-NO   & 106,488,900 & 500 & \textbf{2893.28} & 2.74e-02 \\
			\midrule
			\multirow{5}{*}{PB ($k=1$)}
			& DeepONet & 1,024,513 & 5000 & 1685.50 & 6.89e-02 \\
			& FNO      & 11,138,266 & 500 & 3700.89 & 8.29e-02 \\
			& UNO      & 11,212,122 & 500 & 5602.36 & 9.10e-02 \\
			& Transolver & 4,228,730 & 500 & 32475.36 & 8.30e-02\\
			& LNF-NO   & 8,027,156 & 500 & \textbf{671.30} & \textbf{1.99e-02} \\
			\bottomrule
		\end{tabular}
	\end{sc}
\end{table*}

The quantitative results are summarized in Table~\ref{tab:baseline_single}.
Across these single-input benchmarks, LNF-NO consistently exhibits a clear advantage in training efficiency, and it attains the best accuracy on Laplace, Darcy flow with smooth coefficients, and the source-free Poisson--Boltzmann benchmark.
On the Darcy flow (Smooth) benchmark, for example, LNF-NO achieves the lowest error ($1.79\times 10^{-3}$) among the compared models, while requiring substantially less training time than the competing baselines.

This optimization advantage is further illustrated in Fig.~\ref{fig:pb1_convergence}, where LNF-NO shows markedly faster error reduction on the source-free PB benchmark, both as a function of training epoch and of wall-clock time.

\begin{figure*}[t]
	\centering
	\includegraphics[width=0.48\linewidth]{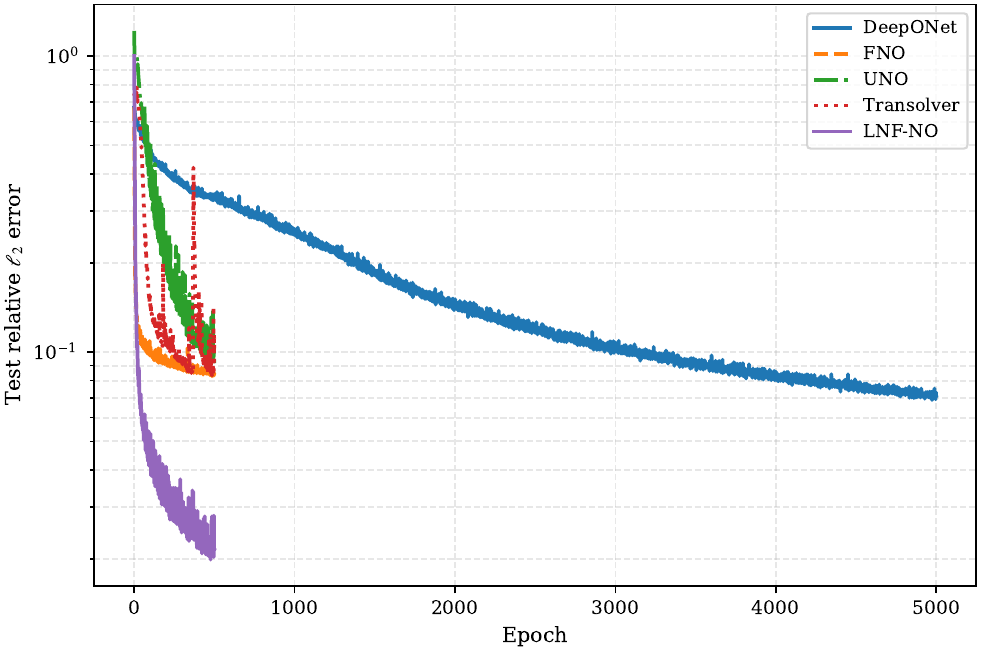}
	\hfill
	\includegraphics[width=0.48\linewidth]{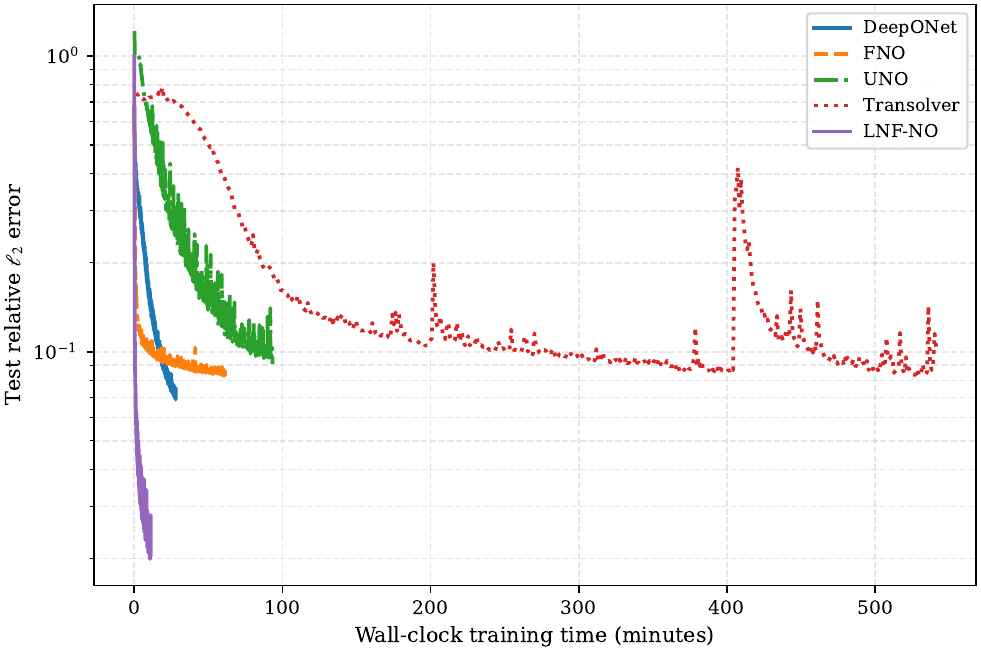}
	\caption{\textbf{Convergence comparison on the source-free Poisson--Boltzmann benchmark ($k=1$).}
		Left: test relative $\ell_2$ error versus training epoch.
		Right: test relative $\ell_2$ error versus wall-clock training time.
		LNF-NO exhibits consistently faster error reduction than the compared baselines, especially when measured against actual training time.}
	\label{fig:pb1_convergence}
\end{figure*}

\subsubsection{Multi-input operators}

We next consider \emph{multi-input} operator-learning problems,
where the input consists of multiple function-valued components.
These include a source-driven Poisson--Boltzmann equation
and a coupled Poisson--Nernst--Planck system.

\begin{table*}[t]
	\centering
	\caption{Baseline benchmarks on regular grids (multi-input operators).
		Each task takes multiple function-valued inputs (e.g., boundary traces and a source field) and predicts fields on a fixed grid.}
	\label{tab:baseline_multi}
	\begin{sc}
		\renewcommand{\arraystretch}{1.0}
		\begin{tabular}{l l c c c c}
			\toprule
			Equation & Model & \#Params & Epochs & Training Time (s) & Test rel.\ $\ell_2$ \\
			\midrule
			\multirow{5}{*}{PB with source ($k=1$)}
			& DeepONet & 4,096,769 & 5000 & 2549.46 & 2.64e-01 \\
			& FNO      & 11,138,330 & 500 & 3751.44 & 3.71e-02 \\
			& UNO      & 11,212,186 & 500 & 5683.71 & 3.63e-02 \\
			& Transolver & 4,228,986 & 500 & 28264.52 & 4.85e-02\\
			& LNF-NO   & 9,662,852 & 500 & \textbf{988.78} & \textbf{1.47e-02} \\
			\midrule
			\multirow{5}{*}{PNP system}
			& DeepONet & 1,841,667 & 5000 & 12299.65 & 6.01e-02 \\
			& FNO      & 8,409,987 & 500 & 8197.63 & 2.21e-02 \\
			& UNO      & 8,483,843 & 500 & 9873.26 & 3.64e-02 \\
			& Transolver & 1,439,395 & 500 & 47281.55 & \textbf{8.19e-03}\\
			& LNF-NO   & 35,954,602 & 500 & \textbf{2872.09} & 1.88e-02 \\
			\bottomrule
		\end{tabular}
	\end{sc}
\end{table*}

Table~\ref{tab:baseline_multi} presents the performance on these multi-input tasks.
LNF-NO naturally accommodates multi-function inputs through independent encoders and a shared operator core.
For the source-driven PB equation, LNF-NO attains the best accuracy among the compared models ($1.47\times 10^{-2}$) while requiring substantially less training time than FNO, UNO, and Transolver.
For the coupled PNP system, the lowest error is achieved by Transolver, whereas LNF-NO remains competitive in accuracy and offers a clear advantage in training efficiency.

\subsection{Varying Nonlinear Strengths in Poisson--Boltzmann Equations}

We further investigate robustness under increasing nonlinearity
using the source-free Poisson--Boltzmann equation with
$k=0.01$, $1$, and $100$.
As $k$ increases, the operator becomes increasingly stiff
and exhibits sharp boundary layers.

\begin{table*}[t]
	\centering
	\caption{Poisson--Boltzmann regimes with varying nonlinear strengths (source-free).
		We vary the nonlinearity strength $k$ in \eqref{eq:pb} with $f=0$ and report test mean relative $\ell_2$ error and training time.}
	\label{tab:nonlinear}
	\begin{sc}
		\renewcommand{\arraystretch}{1.0}
		\begin{tabular}{l l c c c c}
			\toprule
			Equation & Model & \#Params & Epochs & Training Time (s) & Test rel.\ $\ell_2$ \\
			\midrule
			\multirow{5}{*}{PB ($k=0.01$)}
			& DeepONet & 1,024,513 & 5000 & 1821.46 & 5.72e-02 \\
			& FNO      & 11,138,266 & 500 & 3747.64 & 6.24e-02 \\
			& UNO      & 11,212,122 & 500 & 5862.01 & 6.76e-02 \\
			& Transolver & 4,228,730 & 500 & 28642.89 & 4.27e-02\\
			& LNF-NO   & 8,027,156 & 500 & \textbf{668.07} & \textbf{1.65e-02} \\
			\midrule
			\multirow{5}{*}{PB ($k=1$)}
			& DeepONet & 1,024,513 & 5000 & 1685.50 & 6.89e-02 \\
			& FNO      & 11,138,266 & 500 & 3700.89 & 8.29e-02 \\
			& UNO      & 11,212,122 & 500 & 5602.36 & 9.10e-02 \\
			& Transolver & 4,228,730 & 500 & 32475.36 & 8.30e-02\\
			& LNF-NO   & 8,027,156 & 500 & \textbf{671.30} & \textbf{1.99e-02} \\
			\midrule
			\multirow{5}{*}{PB ($k=100$)}
			& DeepONet & 1,024,513 & 5000 & 2280.09 & 1.01e-01 \\
			& FNO      & 11,138,266 & 500 & 3735.93 & 1.01e-01 \\
			& UNO      & 11,212,122 & 500 & 5901.40 & 1.18e-01 \\
			& Transolver & 4,228,730 & 500 & 28294.86 & 1.07e-01\\
			& LNF-NO   & 8,027,156 & 500 & \textbf{735.33} & \textbf{2.28e-02} \\
			\bottomrule
		\end{tabular}
	\end{sc}
\end{table*}

The results under varying nonlinearity strengths are detailed in Table~\ref{tab:nonlinear}.
Across all three nonlinearity regimes, LNF-NO attains the lowest test error among the compared models while maintaining a clear advantage in training efficiency.
For the stiffest case ($k=100$), LNF-NO achieves a relative error of $2.28\times 10^{-2}$, significantly outperforming all compared baselines, suggesting that the proposed multiplicative fusion provides an effective inductive bias for stiff nonlinear operators.

\subsection{Irregular geometries}

\begin{table*}[t]
	\centering
	\caption{Performance on irregular domains for the Poisson--Boltzmann equation.
		The problems are posed on non-rectangular geometries with complex boundaries
		and discretized using unstructured node-based representations,
		for which grid-based spectral operators such as FNO are not directly applicable.}
	\label{tab:irregular}
	\begin{sc}
		\renewcommand{\arraystretch}{1.0}
		\begin{tabular}{c l c c c c}
			\toprule
			Geometry & Model & \#Params & Epochs & Training Time (s) & Test rel.\ $\ell_2$ \\
			\midrule
			\multirow{3}{*}{\includegraphics[height=3.5em]{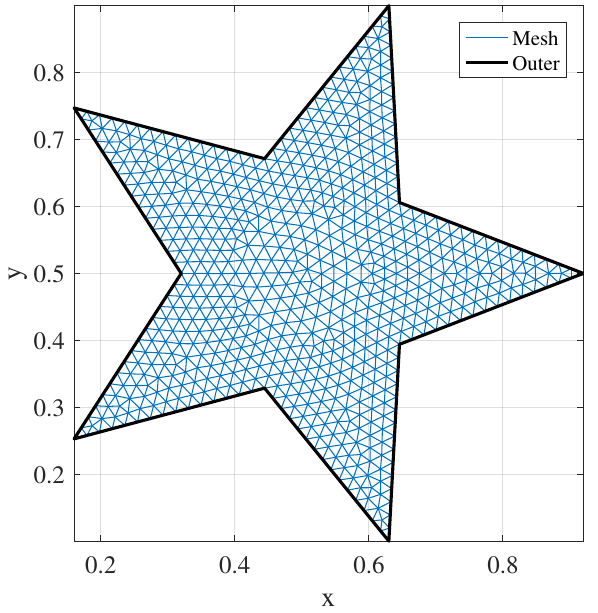}}
			& DeepONet & 973,313 & 5000 & 1050.88 & 3.47e-02 \\
			& Transolver & 1,014,924 & 500 & 3137.79 & 2.92e-02\\
			& LNF-NO   & 1,912,695 & 500 & \textbf{114.79} & \textbf{1.05e-02} \\
			\midrule
			\multirow{3}{*}{\includegraphics[height=3.5em]{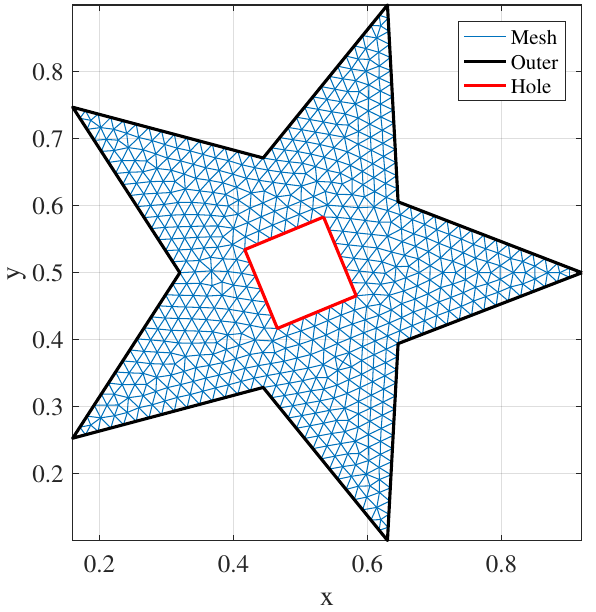}}
			& DeepONet & 973,313 & 5000 & 925.12 & 3.46e-02 \\
			& Transolver & 1,004,644 & 500 & 3012.29 & 2.71e-02\\
			& LNF-NO   & 1,892,135 & 500 & \textbf{109.16} & \textbf{1.12e-02} \\
			\midrule
			\multirow{3}{*}{\includegraphics[height=3.5em]{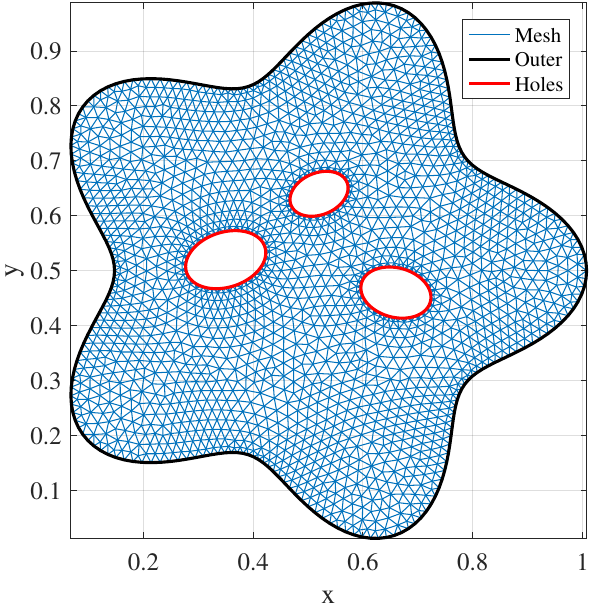}}
			& DeepONet & 973,313 & 5000 & 1025.40 & 3.20e-02 \\
			& Transolver & 1,411,732 & 500 & 10148.53 & 2.70e-02\\
			& LNF-NO   & 2,706,311 & 500 & \textbf{117.66} & \textbf{1.23e-02} \\
			\bottomrule
		\end{tabular}
	\end{sc}
\end{table*}

Table~\ref{tab:irregular} compares the performance on three non-rectangular geometries.
LNF-NO consistently attains the best accuracy on all three irregular geometries while requiring substantially less wall-clock training time than the compared baselines.
On the Star domain, for example, LNF-NO attains a test relative error of $1.05\times 10^{-2}$, substantially improving over both DeepONet and Transolver while requiring much less training time.
These results highlight the geometric generality of our method: LNF-NO remains directly applicable on irregular domains discretized by unstructured nodes, whereas grid-based spectral operators such as FNO are not directly applicable in this setting.

\subsection{Extension to three-dimensional problems}

Quantitative comparisons for the 3D benchmark are provided in Table~\ref{tab:3d_pb}.
Although Transolver attains the lowest error, LNF-NO achieves strong 3D accuracy ($4.32\times 10^{-2}$) while requiring substantially less training time than both 3D FNO and Transolver.
Qualitatively, Figure~\ref{fig:pb3d_main} visualizes the central cross-sections of the predicted solution field $u(x,y,z)$, demonstrating that LNF-NO can accurately reconstruct the 3D potential distribution from boundary data.

\begin{table}[t]
	\centering
	\caption{Results on the 3D Poisson--Boltzmann benchmark on a $33^3$ grid.
		We report test mean relative $\ell_2$ error and wall-clock training time under the unified protocol.}
	\label{tab:3d_pb}
	\begin{sc}
		\resizebox{0.6\linewidth}{!}{
			\begin{tabular}{l c c c c}
				\toprule
				Model & \#Params & Epochs & Time (s) & Rel.\ $\ell_2$ \\
				\midrule
				DeepONet & 2.5M  & 5000 & 1349 & 3.34e-01 \\
				FNO (3D) & 15.0M & 500  & 7559 & 7.71e-02 \\
				UNO & 13.0M & 500  & 10612 & 8.53e-02 \\
				Transolver & 11.8M & 500  & 53311 & \textbf{3.66e-02} \\
				LNF-NO   & 21.7M & 500  & 3253 & 4.32e-02 \\
				\bottomrule
		\end{tabular}}
	\end{sc}
\end{table}

\begin{figure*}[t]
	\centering
	\includegraphics[width=0.9\linewidth]{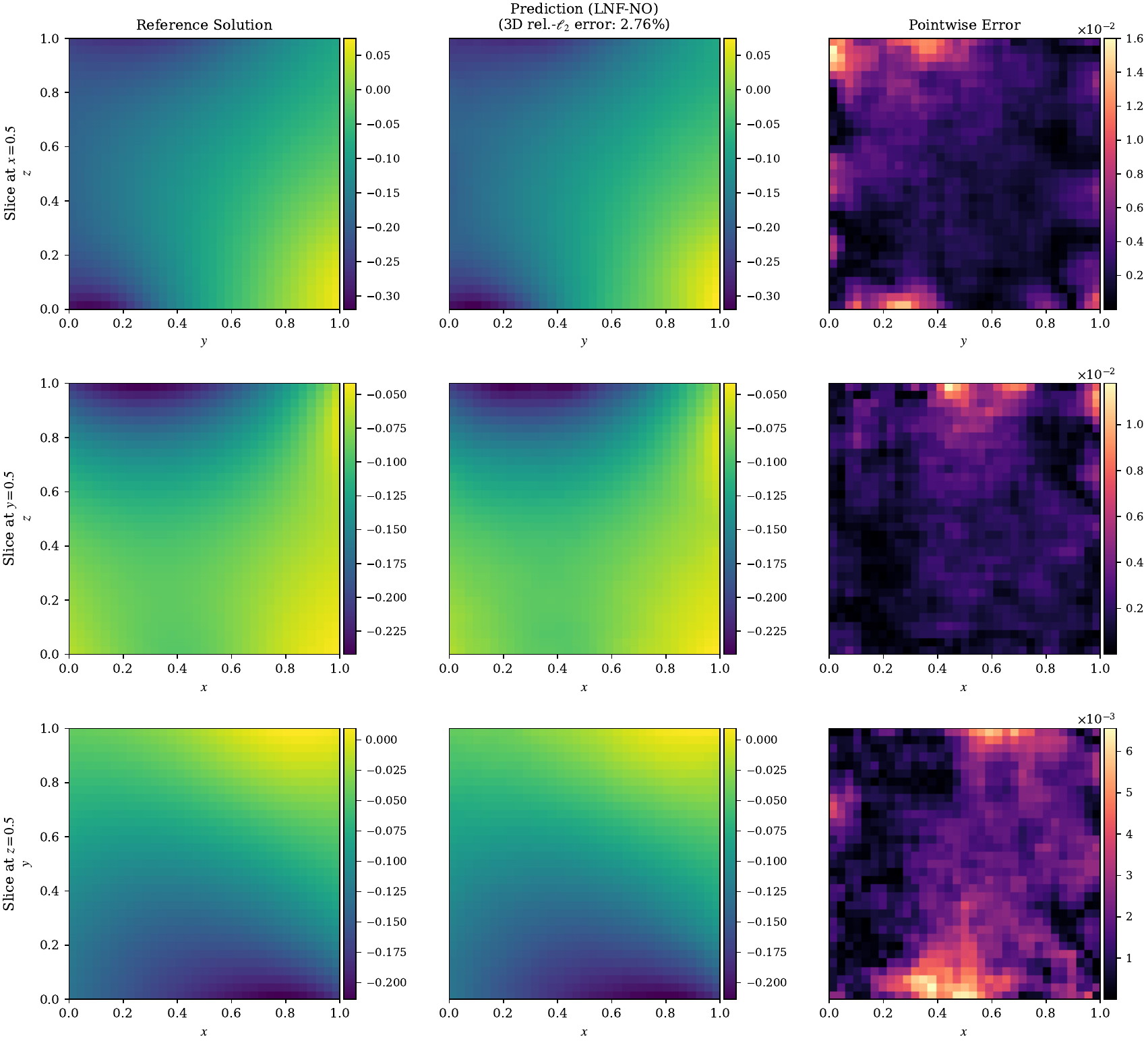}
	\caption{\textbf{3D Poisson--Boltzmann extension.}
		Central cross-sections of the solution field $u(x,y,z)$ at $x=0.5$, $y=0.5$, and $z=0.5$.
		The prediction is reconstructed from boundary data by LNF-NO.}
	\label{fig:pb3d_main}
\end{figure*}

Taken together, the experimental results across Tables~\ref{tab:baseline_single}--\ref{tab:3d_pb} show a consistent pattern.
On regular-grid benchmarks (Tables~\ref{tab:baseline_single} and \ref{tab:baseline_multi}), LNF-NO achieves strong overall performance while substantially reducing wall-clock training time relative to the compared baselines.
Beyond regular grids, as evidenced in Table~\ref{tab:irregular}, LNF-NO remains directly applicable to irregular geometries represented by unstructured nodes, where grid-based spectral operators are not readily usable.
In addition, on the 3D Poisson--Boltzmann benchmark (Table~\ref{tab:3d_pb} and Fig.~\ref{fig:pb3d_main}), LNF-NO maintains stable optimization and achieves strong 3D accuracy while requiring substantially less training time than the 3D FNO and Transolver baselines.
Overall, these results suggest that the proposed multiplicative fusion functions primarily as an efficiency-oriented inductive bias that is broadly compatible with heterogeneous inputs, non-rectangular geometries, and three-dimensional extensions, rather than a task-specific architectural specialization.

	\section{Discussion}
	\label{sec:discussion}
	
	The experimental results indicate that, within the scope of the problem settings and equations investigated, LNF-NO offers a valuable trade-off between efficiency, versatility, and accuracy.
	Rather than claiming superiority across all possible PDE learning tasks, we position LNF-NO as a streamlined, broadly applicable framework that is particularly effective for multi-physics problems and geometric variations, where specialized grid-based or highly task-specific baselines may require additional adaptations.
	
	First, LNF-NO exhibits a high degree of architectural uniformity.
	While grid-based spectral approaches (e.g., FNO) are inherently restricted to regular meshes and require specific adaptations for complex geometries, LNF-NO employs a unified linear--nonlinear fusion core across regular grids, irregular meshes, and three-dimensional domains.
	Although point-based baselines like DeepONet also support irregular geometries, our experiments suggest that they may face optimization difficulties in higher-dimensional or strongly nonlinear settings.
	In contrast, LNF-NO maintains stable optimization and consistent performance across these diverse scenarios using a single architectural paradigm.
	This design naturally accommodates multiple functional inputs and coupled multi-field outputs without requiring problem-specific architectural redesign.
	
	Second, LNF-NO consistently demonstrates significantly improved training efficiency.
	Under the common evaluation protocol used in this work, our method exhibits a clear advantage in training efficiency on most benchmarks.
	The linear--nonlinear fusion structure acts as an efficiency-oriented inductive bias and empirically leads to faster error reduction, particularly in stiff nonlinear regimes.
	Unlike emerging Transformer-based large-scale foundation models~\cite{Ye2024PDEformer} that prioritize massive model capacity and often incur high computational costs, LNF-NO targets an explicitly efficiency-oriented design regime.
	This makes it particularly suitable as a lightweight baseline for rapid exploration and analysis.
	
	Third, regarding predictive performance, LNF-NO achieves comparable or improved accuracy in most cases relative to the baselines.
	It performs particularly well on strongly nonlinear Poisson--Boltzmann problems, irregular-domain benchmarks, and several multi-input settings.
	We also note that on certain established benchmarks, such as Burgers, discontinuous Darcy flow, the PNP system, and the 3D Poisson--Boltzmann problem, the lowest error may be attained by other specialized baselines.
	Even in these cases, however, LNF-NO remains competitive in accuracy while retaining clear advantages in training efficiency and architectural uniformity.
	
	Future work will focus on extending the fusion mechanism to practical engineering contexts involving more complex geometries and realistic boundary conditions.
	We also plan to investigate the scalability of the multiplicative fusion mechanism in higher-resolution three-dimensional simulations, where balancing accuracy and computational efficiency becomes increasingly important.
	
	\section*{Impact Statement}
	
	This work aims to advance machine learning for scientific computing,
	with a particular focus on efficient and general neural operator learning for partial differential equations.
	The proposed LNF-NO architecture improves the training efficiency and practical applicability
	of operator-learning models, which has potential impact in areas such as
	computational physics, engineering design, and computational biology.
	In particular, problems arising from Poisson--Boltzmann and Poisson--Nernst--Planck models
	are central to biomolecular electrostatics and ion transport modeling,
	and more efficient operator surrogates may help accelerate exploratory simulations
	and parameter studies in these domains.
	
	We do not anticipate any direct negative societal or ethical consequences arising from this work.
	The methods presented are intended to support scientific research and modeling,
	and do not introduce new concerns related to misuse, fairness, or safety beyond
	those commonly associated with numerical simulation and surrogate modeling techniques.

	\clearpage
	
	\appendix
	\renewcommand{\thefigure}{A\arabic{figure}}
	\renewcommand{\thetable}{A\arabic{table}}
	\setcounter{figure}{0}
	\setcounter{table}{0}
	\onecolumn
	
	\section{Dataset generation and preprocessing}
	\label{app:data}
	
	We provide detailed descriptions of the data generation process for each benchmark considered in this work.
	Specific numerical parameters are detailed below, and the core implementation scripts are available upon reasonable request.
	
	\subsection{Laplace dataset generation}
	\label{app:laplace_mad}
	
	We consider the 2D Laplace equation
	\begin{equation}
		\Delta u = 0 \quad \text{in } \Omega=[0,1]^2,
		\qquad u|_{\partial\Omega}=g,
	\end{equation}
	and learn the solution operator mapping the Dirichlet boundary trace $g$ to the solution field $u$ on a fixed grid.
	This dataset is generated following the Mathematical Artificial Data (MAD) framework proposed by \cite{Wu2025MAD}.
	
	\paragraph{Discretization.}
	We use a uniform $N\times N$ grid on $\Omega$ with $N=51$ and spacing $h=1/(N-1)$, resulting in $N_g=N^2=2601$ evaluation points.
	The boundary trace is sampled on $\partial\Omega$ using $N$ points per side with duplicated corner points removed, yielding
	$N_b=4(N-1)=200$ boundary samples.
	The boundary points are ordered by concatenating the four edges in a fixed counterclockwise manner.
	
	\paragraph{MAD solution family.}
	For each sample, we synthesize a harmonic function as a linear combination of fundamental solutions with sources located outside the domain.
	Specifically, we draw $J=10$ source points $\{c_j\}_{j=1}^{J}$ outside an expanded box $[-\varepsilon,1+\varepsilon]^2$ with $\varepsilon=10^{-3}$,
	and random weights $\{w_j\}_{j=1}^{J}$.
	The solution is defined by
	\begin{equation}
		u(x) = \sum_{j=1}^{J} w_j \cdot \frac{1}{2}\log\big(\|x-c_j\|^2\big),
	\end{equation}
	which is harmonic inside $\Omega$ because all sources lie outside $\Omega$.
	
	\paragraph{Data format and normalization.}
	Each sample is stored as a single vector
	\[
	\big[u_b \,;\, u\big]\in\mathbb{R}^{N_b+N_g},
	\]
	where $u_b$ contains boundary values evaluated at the $N_b$ boundary points and $u$ contains the solution values on the $N_g$ grid points.
	We apply per-sample normalization by dividing the entire vector by its maximum absolute value so that $\max|[u_b;u]|=1$.
	
	\subsection{Burgers dataset generation}
	\label{app:burgers_lowmodes}
	
	We consider the one-dimensional viscous Burgers equation on a periodic domain,
	\begin{equation}
		\partial_t u(x,t) + u(x,t)\,\partial_x u(x,t) = \nu\,\partial_{xx} u(x,t),
		\qquad x \in [0,2\pi),\; t \in [0,T],
	\end{equation}
	with periodic boundary conditions and viscosity $\nu=0.01$.
	The learning task is to approximate the solution operator mapping the initial condition
	$u_0(x)=u(x,0)$ to the solution field $u(x,t)$ for $t>0$.
	
	\paragraph{Initial condition sampling.}
	Initial conditions are generated as low-frequency random Fourier series,
	\begin{equation}
		u_0(x) = \sum_{m=1}^{K} \frac{a_m \cos(mx) + b_m \sin(mx)}{m^{\alpha}},
	\end{equation}
	where $a_m,b_m \sim \mathcal{N}(0,1)$ are independent standard Gaussian variables,
	$K=8$ controls the maximum frequency, and $\alpha=2.0$ controls spectral decay.
	Each realization is shifted to have zero mean and normalized to unit standard deviation.
	
	\paragraph{Numerical integration.}
	The Burgers equation is integrated on a fine spatial grid ($N_{\text{fine}}=512$) using a Fourier spectral method
	combined with an exponential time-differencing fourth-order Runge--Kutta (ETDRK4) scheme.
	A $2/3$ dealiasing rule is applied in Fourier space to avoid aliasing errors.
	The solution is evolved up to final time $T=1.0$ with time step $\Delta t=10^{-4}$.
	
	\paragraph{Downsampling and dataset construction.}
	To obtain training data on a coarse grid, the high-resolution solution is low-pass filtered in Fourier space
	and then downsampled to $N_x=64$ spatial points.
	The solution is recorded at $N_t=100$ uniformly spaced time snapshots (excluding $t=0$).
	Each dataset sample therefore consists of the coarse-grid initial condition $u_0 \in \mathbb{R}^{64}$ and the solution trajectory $u \in \mathbb{R}^{64 \times 100}$.
	The total dataset size is 2000 samples.
	
	\subsection{Darcy flow datasets}
	\label{app:darcy}
	
	We consider two Darcy flow benchmarks with different regularity properties of the permeability field:
	a discontinuous-coefficient dataset obtained from a publicly released benchmark, and a continuous-coefficient
	dataset generated synthetically by us.
	
	\paragraph{Discontinuous permeability (public dataset).}
	For the discontinuous-permeability benchmark, we use a publicly released Darcy dataset that has been widely
	adopted in the neural operator and Fourier Neural Operator (FNO) literature.
	The dataset consists of paired permeability fields $a(x,y)$ with sharp discontinuities and corresponding
	solutions $u(x,y)$ of the elliptic problem
	\[
	-\nabla\cdot(a(x,y)\nabla u(x,y)) = 1 \quad \text{in } (0,1)^2, \qquad u=0 \text{ on } \partial(0,1)^2.
	\]
	The data are provided on a uniform grid of resolution $241\times241$ and distributed in MATLAB format.
	We utilize the benchmark dataset originally introduced by \cite{li2021fourier}. Specifically, we obtain the data files from the public repository \texttt{raj-brown/fourier\_neural\_operator}\footnote{\url{https://github.com/raj-brown/fourier_neural_operator}} and use them directly without further modification.
	
	\paragraph{Continuous permeability (synthetic dataset).}
	For the continuous-permeability benchmark, we generate the dataset synthetically using the following procedure.
	
	\emph{Coefficient generation.}
	The permeability field $a(x,y)$ is generated by evaluating a randomly initialized
	sinusoidal multilayer perceptron on a uniform grid.
	Specifically, we use a fully connected network of architecture $[2,50,50,1]$ with sine activations,
	\[
	a_{\mathrm{raw}}(x,y) = W_3 \sin\big(W_2 \sin(W_1 [x,y]^\top + b_1) + b_2\big) + b_3,
	\]
	and map it to a strictly positive coefficient via
	$a(x,y) = 0.1 + \mathrm{softplus}(a_{\mathrm{raw}}(x,y))$.
	The network weights are sampled independently for each realization.
	
	\emph{High-resolution PDE solve.}
	Given $a(x,y)$, we solve the Darcy equation
	on a high-resolution grid of size $241\times241$ using a finite-difference discretization with face-averaged coefficients
	and a sparse direct solver.
	
	\emph{Downsampling.}
	Both the coefficient field $a(x,y)$ and the solution $u(x,y)$
	are downsampled to a $129\times129$ grid using bilinear interpolation.
	
	\subsection{Poisson--Boltzmann (PB) datasets}
	\label{app:pb}
	
	All PB datasets follow the task definition in Eq.~\eqref{eq:pb}:
	\begin{equation}
		-\Delta u + k\,\sinh(u) = f \quad \text{in } \Omega,
		\qquad u|_{\partial\Omega}=g,
	\end{equation}
	where $k>0$ partially controls the strength of the nonlinearity.
	The \emph{PB with source} case corresponds to $f\neq 0$, while other PB cases use $f=0$.
	
	\subsubsection{PB on unit square without source}
	\label{app:pb_square_fd}
	
	We consider the source-free case $f=0$ on $\Omega=(0,1)^2$.
	The domain is discretized by a uniform $N\times N$ grid ($N=101$) using the finite difference method (5-point stencil).
	Boundary values $g$ are sampled on $\partial\Omega$ ($N_b=400$) and generated as a mixture of low-frequency and high-frequency Gaussian random fields.
	The resulting discrete nonlinear system is solved using a continuation (homotopy) strategy coupled with a Newton--Raphson method, utilizing a direct sparse solver for the linearized steps.
	Each accepted realization is stored as a vector $[g \,;\, u]\in\mathbb{R}^{N_b+N^2}$.
	
	\subsubsection{PB on unit square with source (MAD dataset)}
	\label{app:pb_source_mad}
	
	We also consider a source-driven PB dataset following the MAD paradigm \cite{Wu2025MAD}.
	For each sample, we synthesize a smooth field $u(x,y)$ via a randomized sine network and compute the consistent forcing
	$\tilde f(x,y)=\Delta u - k\sinh(u)$.
	Each sample is stored as $[u_b \,;\, \tilde f \,;\, u]$, where inputs are the boundary trace $u_b$ and the forcing term $\tilde f$.
	
	\subsubsection{Irregular-domain PB datasets}
	\label{app:pb_irregular_fem}
	
	To evaluate geometric generalization, we generate PB datasets on irregular planar domains using finite-element methods (FEM).
	Unlike grid-based datasets, these are \emph{node-based}: solutions are stored at FEM mesh nodes.
	
	\paragraph{Geometry and Solver.}
	We generate three benchmarks: (1) a star-shaped polygon, (2) a star with a square hole, and (3) a flower shape with elliptical holes (see Fig.~\ref{fig:pb_fem_meshes}).
	For each geometry, we generate a fixed triangular mesh.
	Boundary values are sampled from a periodic Gaussian random field and mapped to boundary nodes.
	The nonlinear system is solved using linear (P1) finite elements on the unstructured mesh.
	To ensure robust convergence, we employ a homotopy continuation strategy coupled with a damped Newton--Raphson solver.
	Within each continuation step, the solution is initialized via Picard iterations, and the linearized systems are solved using a direct sparse solver.
	
	\begin{figure}[t]
		\centering
		\includegraphics[width=0.32\linewidth]{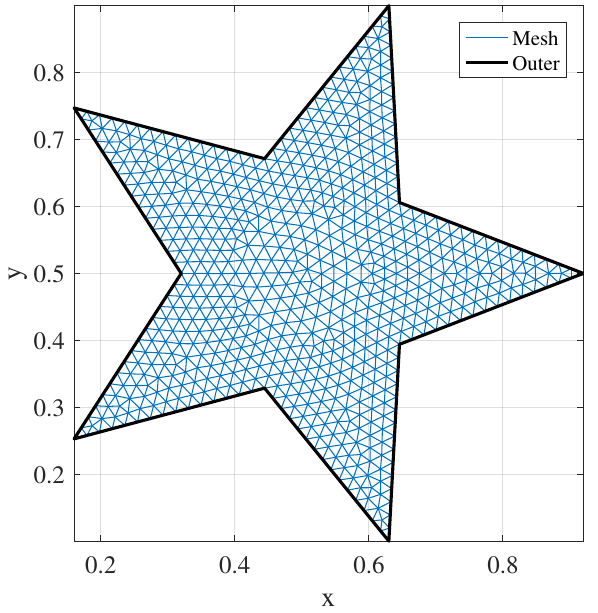}\hfill
		\includegraphics[width=0.32\linewidth]{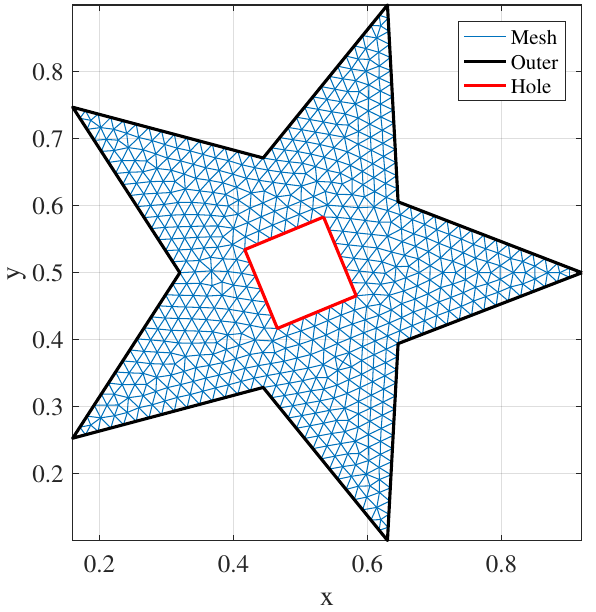}\hfill
		\includegraphics[width=0.32\linewidth]{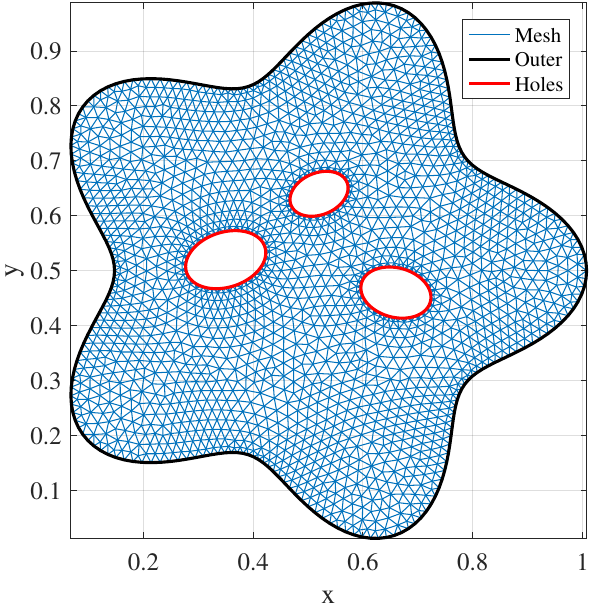}
		\caption{Representative geometries and corresponding finite-element meshes for the irregular-domain Poisson--Boltzmann benchmarks.}
		\label{fig:pb_fem_meshes}
	\end{figure}
	
	\subsubsection{3D PB on unit cube}
	\label{app:pb_3d}
	
	We generate a 3D dataset on $\Omega=[0,1]^3$ with $k=1.0$.
	The domain is discretized into a $33 \times 33 \times 33$ grid ($N_g=35937$).
	Dirichlet boundary conditions are sampled on the 6 faces of the cube, resulting in $N_b = 6N^2 - 12N + 8$ unique boundary points.
	Inputs $g$ are synthesized using random Yukawa potential sources placed outside the domain to ensure smooth but non-trivial boundary patterns.
	The nonlinear system is solved using a finite-difference discretization combined with a damped Newton--Raphson method, where the linearized systems are solved using a direct sparse solver.
	The dataset consists of 2000 samples.
	
	\subsection{Poisson--Nernst--Planck (PNP) dataset generation}
	\label{app:pnp}
	
	\textbf{Governing Equations.}
	For operator-learning benchmarking, we consider the dimensionless steady-state Poisson--Nernst--Planck system for a binary symmetric electrolyte. The system describes the coupling between the electrostatic potential $\phi$ and the ionic concentrations $c_\pm$ (see, e.g., \cite{zhang2020model} for formulation details), governed by the following equations:
	\begin{equation}
		-\Delta \phi = c_+ - c_-, \qquad \nabla \cdot (\nabla c_\pm \pm c_\pm \nabla \phi) = 0,
	\end{equation}
	where the domain is the unit square $\Omega=(0,1)^2$. The first equation (Poisson) relates the potential to the charge density difference, while the second pair (Nernst--Planck) describes the conservation of mass under diffusion and electromigration. This formulation captures the essential nonlinear coupling where potential gradients drive ion transport and ion distributions determine the potential field.
	
	\textbf{Solver \& Data.}
	We generate the reference solutions using a Gummel fixed-point iteration scheme. At each iteration step, the Poisson equation and the Nernst--Planck equations are solved sequentially using finite differences on a $129 \times 129$ uniform grid until convergence.
	The input boundary conditions are generated via truncated Fourier series, with strictly enforced positivity ($c_\pm > 0$) to ensure physical validity.
	The final dataset consists of 2000 samples, mapping the three boundary traces ($g_\phi, g_{c_+}, g_{c_-}$) to the full interior fields $(\phi, c_+, c_-)$.
	
	\section{Training and evaluation protocol}
	\label{app:train_protocol}
	
	\paragraph{Hardware and software.}
	Experiments are conducted on a single NVIDIA GeForce RTX 4070 Laptop GPU using PyTorch 2.4.1.
	Wall-clock training time excludes data generation.
	
	\paragraph{Data split and normalization.}
	We use a fixed $9{:}1$ train/test split.
	All inputs and outputs are normalized using global statistics (mean and standard deviation) computed strictly on the training set.
	Evaluation is always performed on the \emph{physical scale} by decoding predictions back to the original domain.
	
	\paragraph{Loss functions.}
	For scalar fields, we use the relative $\ell_2$ error:
	\[
	\mathrm{rel}\,\ell_2(\hat{u},u)=\frac{\|\hat{u}-u\|_2}{\|u\|_2+\varepsilon}, \quad \varepsilon=10^{-12}.
	\]
	For the coupled PNP task, we use the mean of the relative errors of the three fields $(\phi, c_+, c_-)$.
	
	\paragraph{Optimization.}
	We use the AdamW optimizer with weight decay $10^{-4}$ (excluding bias and scale parameters).
	DeepONet is trained for 5000 epochs, while the other neural operator baselines and LNF-NO are trained for 500 epochs.
	Learning rates and batch sizes are task- and model-dependent, and are specified in the corresponding experiment scripts.
	
	\paragraph{Reproducibility.}
	Due to computational constraints and to facilitate reproducibility, our main results are reported using a single fixed random seed.
	To assess sensitivity to initialization, we additionally repeat the source-free PB benchmark ($k=1$) across five random seeds and report the corresponding results in Table~\ref{tab:seed_stability_pb1}.
	
	\paragraph{Supplementary code.}
	The implementation of LNF-NO and the core baseline scripts are available upon reasonable request.
	
	\subsection{Seed stability on the source-free PB benchmark}
	\label{app:seed_stability}
	
	To examine the sensitivity of LNF-NO to random initialization, we repeat the source-free Poisson--Boltzmann benchmark with $k=1$ across five random seeds.
	For each run, we report the best test relative $\ell_2$ error and the total wall-clock training time.
	
	\begin{table}[t]
		\centering
		\caption{Seed stability on the source-free Poisson--Boltzmann benchmark ($k=1$).
			We repeat the LNF-NO experiment across five random seeds and report the best test relative $\ell_2$ error and total training time.}
		\label{tab:seed_stability_pb1}
		\begin{sc}
			\renewcommand{\arraystretch}{1.1}
			\begin{tabular}{c c c}
				\toprule
				Seed & Best test rel.\ $\ell_2$ & Total training time (s) \\
				\midrule
				0 & 1.99e-02 & 671.30 \\
				1 & 1.88e-02 & 588.14 \\
				2 & 1.88e-02 & 559.15 \\
				3 & 2.06e-02 & 680.43 \\
				4 & 2.09e-02 & 678.09 \\
				\midrule
				Mean $\pm$ Std. & $1.98\times 10^{-2} \pm 8.8\times 10^{-4}$ & $635.42 \pm 51.35$ \\
				\bottomrule
			\end{tabular}
		\end{sc}
	\end{table}
	
	The results indicate that LNF-NO is reasonably robust to random initialization on this benchmark, with only modest variation in both final accuracy and total training time across the tested seeds.
	
\section{LNF-NO architecture details}
\label{app:lnfno_arch}

We categorize the architectures based on their input-output configurations:
\textbf{Single-Input Single-Output (SISO)}, mapping a single functional input to a single output field;
\textbf{Multiple-Input Single-Output (MISO)}, processing multiple functional inputs to predict a single output field;
and \textbf{Multiple-Input Multiple-Output (MIMO)}, processing multiple functional inputs to predict multiple coupled output fields.

All LNF-NO variants share the same multiplicative fusion structure:
\[
\hat{y} = \alpha \,( \mathcal{B}_L(z) \odot \mathcal{B}_N(z) ),
\]
where an encoder maps the input to latent features $z$, and $\alpha$ is a learned scalar.
The linear branch $\mathcal{B}_L$ is a pure linear map (affine transformation), while the nonlinear branch $\mathcal{B}_N$ uses GELU activations.
Unless otherwise noted, intermediate convolutional layers in the encoders and decoders are also followed by GELU activations.

\subsection{Architecture Configurations}

\paragraph{(A) Regular-domain SISO (Laplace, PB).}
Input: Boundary $g\in\mathbb{R}^{N_b}$. Output: Grid $u\in\mathbb{R}^{N^2}$.
\begin{itemize}
	\item \textbf{Encoder (1D convolutional neural network, CNN):} Four layers.
	\begin{itemize}
		\item \texttt{Conv1d(1, 96, k=9, s=1, p=4)}
		\item Three layers of \texttt{Conv1d(96, 96, k=9, s=2, p=4)}
	\end{itemize}
	\item \textbf{Branches (MLP):} Feature dim $d$ to grid dim $D=N^2$.
	\begin{itemize}
		\item $\mathcal{B}_L$: Linear $[d \to 256 \to D]$. Note that while mathematically equivalent to a single linear map, the two-layer structure (without intermediate activation) is maintained to align the latent width with the nonlinear branch for elementwise fusion and to provide additional optimization degrees of freedom.
		\item $\mathcal{B}_N$: MLP $[d \to 256 \to 256 \to D]$ with GELU activations.
	\end{itemize}
	\item \textbf{Decoder (2D CNN):}
	\begin{itemize}
		\item \texttt{Conv2d(1, 32, k=3, s=1, p=1)}
		\item \texttt{Conv2d(32, 32, k=3, s=1, p=1)}
		\item \texttt{Conv2d(32, 1, k=3, s=1, p=1)}
	\end{itemize}
\end{itemize}

\paragraph{(B) Regular-domain MISO (PB-Source).}
Input: Boundary $g$ and Source $f$. Output: Grid $u$.
\begin{itemize}
	\item \textbf{Encoders:}
	\begin{itemize}
		\item Boundary: Same as (A).
		\item Source: Four-layer 2D CNN with channels $1\to48$, followed by Adaptive Avg Pool to $8\times8$.
	\end{itemize}
	\item \textbf{Branches \& Decoder:} Identical to (A).
\end{itemize}

\paragraph{(C) Irregular-domain SISO (PB-FEM).}
Input: Boundary $g$. Output: Node vector $u \in \mathbb{R}^{N_p}$.
\begin{itemize}
	\item \textbf{Encoder:} Same 1D boundary encoder as (A).
	\item \textbf{Branches:} Output dimension $D=N_p$ (mesh nodes).
	\begin{itemize}
		\item $\mathcal{B}_L$: Linear $[d \to 256 \to D]$.
		\item $\mathcal{B}_N$: MLP $[d \to 256 \to 256 \to D]$ with GELU activations.
	\end{itemize}
	\item \textbf{Decoder:} None (direct node prediction).
\end{itemize}

\paragraph{(D) Regular-domain MIMO (PNP).}
Input: 3 Boundaries. Output: 3 Fields.
\begin{itemize}
	\item \textbf{Encoders:} Three independent 1D encoders, each with the same architecture as in (A).
	\item \textbf{Branches:} Output dim $D=3N_xN_y$.
	\item \textbf{Decoder:} 3-channel 2D CNN (\texttt{3->64->64->3}).
\end{itemize}

\paragraph{(E) 3D Poisson--Boltzmann (Regular-domain SISO).}
Input: Boundary $g \in \mathbb{R}^{N_b}$. Output: Grid $u \in \mathbb{R}^{N^3}$ ($N=33$).
Unlike the 2D cases, we omit the explicit 1D encoder for the boundary vector to reduce memory overhead, passing $g$ directly to the branches.
\begin{itemize}
	\item \textbf{Branches (MLP):} Map input dimension $N_b$ directly to grid dimension $D=N^3$.
	\begin{itemize}
		\item $\mathcal{B}_L$: Linear projection $[N_b \to 256 \to D]$.
		\item $\mathcal{B}_N$: MLP $[N_b \to 256 \to 256 \to D]$ with GELU activations.
	\end{itemize}
	\item \textbf{Decoder (3D CNN):}
	\begin{itemize}
		\item \texttt{Conv3d(1, 32, k=3, s=1, p=1)}
		\item \texttt{Conv3d(32, 32, k=3, s=1, p=1)}
		\item \texttt{Conv3d(32, 1, k=3, s=1, p=1)}
	\end{itemize}
\end{itemize}
	
	\section{Baselines}
	\label{app:baselines}
	
	To benchmark the performance of LNF-NO, we compare it against several representative neural operator baselines, including the Deep Operator Network (DeepONet) \cite{Lu_2021}, the Fourier Neural Operator (FNO) \cite{li2021fourier}, the U-shaped Neural Operator (UNO) \cite{Rahman2023UNO}, and Transolver \cite{Wu2024Transolver}.
	The hyperparameters selected for the baselines (e.g., basis dimension $p$ and frequency modes $k_{\max}$) align with standard configurations widely used in the community to ensure a fair comparison.
	All baseline models are trained using the same general protocol as LNF-NO:
	\begin{itemize}
		\item \textbf{Data Split:} A fixed $9{:}1$ training-to-testing ratio is used for all datasets.
		\item \textbf{Normalization:} Inputs and outputs are normalized using global statistics computed strictly on the training set.
		\item \textbf{Metrics:} All reported errors are relative $\ell_2$ norms computed on the decoded physical scale after inverse normalization.
		\item \textbf{Optimization:} We use the AdamW optimizer, with bias terms excluded from weight decay. Learning rates, batch sizes, and training schedules are selected according to the model type and task, while remaining consistent with the evaluation protocol described in Appendix~\ref{app:train_protocol}.
	\end{itemize}
	DeepONet models are trained with a larger epoch budget (5000 epochs), whereas the other baselines are trained for 500 epochs, reflecting their different optimization characteristics.
	
	\subsection{Deep Operator Network (DeepONet)}
	
	We employ the unstacked DeepONet architecture, which approximates the operator $\mathcal{G}$ by taking an input function $u$ (via a discrete sensor vector) and a query coordinate $y$ to predict the solution value $\mathcal{G}(u)(y)$.
	The network consists of two sub-networks: a \emph{Branch Net} encoding the input function and a \emph{Trunk Net} encoding the coordinates.
	The final prediction is given by the inner product of their outputs (plus an optional bias):
	\[
	\hat{u}(y) = \sum_{k=1}^{p} b_k(u) \cdot t_k(y) + \beta,
	\]
	where $\{b_k\}_{k=1}^p$ and $\{t_k\}_{k=1}^p$ are the $p$-dimensional outputs of the branch and trunk networks, respectively.
	For all tasks, we set the basis dimension to $p=1024$.
	The Branch and Trunk networks are implemented as 5-layer MLPs with a hidden width of 256 and \texttt{tanh} activations.
	
	\paragraph{(A) Regular-domain SISO (Laplace, PB).}
	The branch net input is the boundary vector $g \in \mathbb{R}^{N_b}$, and the trunk net input is the coordinate $x \in \mathbb{R}^2$.
	\begin{itemize}
		\item Branch: $\texttt{MLP}(N_b \to 1024)$ with hidden width 256.
		\item Trunk: $\texttt{MLP}(2 \to 1024)$ with hidden width 256.
	\end{itemize}
	
	For the 3D PB task, the trunk net takes 3D coordinates $(x,y,z)$ as input. The architecture parameters remain consistent ($p=1024$, width=256).
	
	\paragraph{(B) Regular-domain MISO (PB-Source).}
	For tasks with multiple inputs (boundary $u_b \in \mathbb{R}^{N_b}$ and source $f \in \mathbb{R}^{N^2}$), we employ a multi-branch fusion strategy (similar to MIONet \cite{MIONet}).
	Separate branch nets encode $u_b$ and $f$, and their embeddings are fused via elementwise multiplication before the dot product with the trunk:
	\[
	b(u_b, f) = b^{(u)}(u_b) \odot b^{(f)}(f) \in \mathbb{R}^{p}.
	\]
	\begin{itemize}
		\item Branch 1: $\texttt{MLP}(N_b \to 1024)$.
		\item Branch 2: $\texttt{MLP}(N^2 \to 1024)$.
		\item Trunk: $\texttt{MLP}(2 \to 1024)$.
	\end{itemize}
	
	\paragraph{(C) Regular-domain MIMO (PNP).}
	For the coupled PNP system, we use a single concatenated branch net. The trunk net is modified to have three separate output heads to predict the three fields $(\phi, c_+, c_-)$ simultaneously:
	\[
	\hat{u}^{(m)}(x) = \sum_{k=1}^{p} b_k(\mathbf{u}_{\text{bc}}) \cdot t^{(m)}_k(x), \quad m \in \{\phi, c_+, c_-\}.
	\]
	\begin{itemize}
		\item Branch: $\texttt{MLP}(3N_b \to 1024)$.
		\item Trunk: $\texttt{MLP}(2 \to 3 \times 1024)$.
	\end{itemize}
	
	\subsection{Fourier Neural Operator (FNO)}
	
	We adopt the standard Fourier Neural Operator architecture.
	The model processes the input $a(x)$ through three main stages:
	\begin{enumerate}
		\item \textbf{Encoder (Lifting):} The input is projected pointwise to a high-dimensional channel space $v_0(x) \in \mathbb{R}^{d_h}$ via a linear layer or shallow MLP.
		\item \textbf{Fourier Integral Layers:} The latent representation is updated through a sequence of $L$ layers. The update rule for the $l$-th layer is:
		\[
		v_{l+1}(x) = \sigma\left( W_l v_l(x) + (\mathcal{F}^{-1} \circ R_l \circ \mathcal{F})(v_l)(x) \right),
		\]
		where $\sigma$ is a nonlinear activation function, $W_l$ is a pointwise linear transform (residual connection), $\mathcal{F}$ denotes the FFT, and $R_l$ is a learnable complex weight tensor that mixes the lowest $k_{\max}$ frequency modes.
		\item \textbf{Decoder (Projection):} The final features $v_L(x)$ are mapped to the target dimension using a pointwise MLP decoder.
	\end{enumerate}
	
	\paragraph{Input lifting strategy.}
	Since FNO requires grid-structured inputs, strictly lower-dimensional inputs (e.g., boundary conditions) are mapped to the full spatiotemporal grid prior to the lifting stage.
	We employ standard techniques consistent with the literature: \emph{domain replication} for scalar/vector inputs, \emph{zero-padding} (placing boundary values on the perimeter with zero-filled interiors) for boundary fields, or learned \emph{MLP projection}. The resulting fields are concatenated with coordinate grids.
	
	\paragraph{Hyperparameters.}
	We set the network depth to $L=4$ for all tasks.
	Consistent with our training protocol, we employ ReLU activation for all FNO baselines.
	Specific structural parameters are:
	
	\begin{itemize}
		\item \textbf{2D steady-state (Laplace, PB, Darcy, PNP):}
		We use a 2D FNO with $k_{\max}=16$ modes per direction and width $d_h=64$.
		The PNP model uses a 5-channel input (via zero-padding) and predicts 3 fields simultaneously.
		
		\item \textbf{Burgers (1D time-dependent):}
		Modeled as a 2D FNO on the $(x,t)$ domain.
		Configuration: $k_{\max}=16$, $d_h=64$.
		
		\item \textbf{3D Poisson--Boltzmann:}
		Modeled as a 3D FNO on the $(x,y,z)$ domain.
		Configuration: $k_{\max}=8$ modes per direction, width $d_h=32$.
	\end{itemize}
	
	\subsection{U-shaped Neural Operator (UNO)}
	\label{app:uno}
	
	We adopt the U-shaped Neural Operator (UNO)~\cite{Rahman2023UNO} as a neural operator baseline.
	UNO is a multi-resolution neural operator architecture that combines global operator learning with a U-shaped encoder--decoder representation.
	Unlike the standard FNO, which operates at a fixed grid resolution, UNO maps features across multiple grid resolutions, allowing the model to capture both global interactions and fine-scale spatial structures.
	
	The architecture consists of three main components:
	\begin{enumerate}
		\item \textbf{Input lifting:}
		The grid-based input field is concatenated with spatial or spatiotemporal coordinate channels and projected pointwise to a higher-dimensional feature representation.
		
		\item \textbf{Multi-resolution operator blocks:}
		The latent representation is processed by a hierarchy of operator blocks arranged along an encoder--decoder path.
		The encoder path maps the representation to progressively coarser grids, while the decoder path maps it back to the target output resolution.
		Each block applies a learned integral operator, implemented in the Fourier domain, together with pointwise channel mixing.
		The encoder--decoder hierarchy transfers features between different spatial resolutions, and skip connections pass high-resolution representations from the encoder path to the decoder path, mitigating information loss caused by the coarse-grid bottleneck.
		
		\item \textbf{Output projection:}
		After the representation is mapped to the target grid resolution, the final feature map is projected to the target solution field by a pointwise MLP.
	\end{enumerate}
	
	\paragraph{Input representation.}
	Following the FNO setting, all inputs are represented as grid-based fields.
	For coefficient-driven problems such as Darcy flow, the coefficient field is used directly as the input field.
	For boundary-driven problems such as the Poisson--Boltzmann and Laplace equations, the lower-dimensional boundary data are first embedded into grid-based input fields using the same boundary-to-grid preprocessing strategy as in the corresponding FNO baseline.
	For source-driven Poisson--Boltzmann problems, the grid-embedded boundary field and the discretized source field are concatenated as input channels.
	For multi-field systems such as the Poisson--Nernst--Planck equation, multiple grid-embedded boundary fields are used as input channels together with spatial coordinate information, and the model predicts multiple solution fields simultaneously.
	
	\paragraph{Hyperparameters.}
	For each task group, we report the maximum retained Fourier modes and the base channel width.
	Whenever possible, these settings are kept consistent with the corresponding FNO baseline so that the comparison mainly reflects architectural differences rather than hyperparameter changes.
	The configurations are summarized as follows:
	\begin{itemize}
		\item \textbf{2D steady-state problems (Darcy, Poisson--Boltzmann, Laplace):}
		We use a 2D UNO with up to $k_{\max}=16$ retained Fourier modes per spatial dimension and base width $d_h=64$.
		
		\item \textbf{2D Poisson--Boltzmann with source:}
		We use the same 2D UNO configuration with up to $k_{\max}=16$ retained Fourier modes per spatial dimension and base width $d_h=64$.
		The input contains both the grid-embedded boundary information and the discretized source field.
		
		\item \textbf{2D Poisson--Nernst--Planck (PNP):}
		We use a 2D UNO with up to $k_{\max}=16$ retained Fourier modes per spatial dimension and base width $d_h=64$.
		The input consists of the grid-embedded boundary fields for $\phi$, $c_+$, and $c_-$, together with spatial coordinate information.
		The model predicts the three solution fields simultaneously.
		
		\item \textbf{Burgers equation (1D time-dependent):}
		The problem is represented on the $(x,t)$ domain and modeled using a 2D UNO with up to $k_{\max}=16$ retained Fourier modes per dimension and base width $d_h=64$.
		
		\item \textbf{3D Poisson--Boltzmann:}
		We use a 3D UNO on the $(x,y,z)$ domain with up to $k_{\max}=8$ retained Fourier modes per spatial dimension and base width $d_h=32$.
	\end{itemize}
	
	\subsection{Transolver}
	\label{app:transolver}
	
	We adopt Transolver~\cite{Wu2024Transolver} as an attention-based neural operator baseline.
	Unlike DeepONet, which evaluates the solution at query coordinates through branch--trunk inner products, and unlike FNO, which relies on spectral convolution on regular grids, Transolver treats the computational domain as a set of tokens and updates token features through physics-attention blocks.
	For structured-grid problems, the tokens correspond to points on the target Cartesian grid.
	For irregular-domain finite-element problems, the tokens correspond to FEM mesh nodes.
	
	\paragraph{Backbone.}
	For the two-dimensional structured-grid benchmarks, we use a structured-mesh Transolver backbone.
	Given token coordinates $x_i\in\mathbb{R}^2$ and optional functional features $f_i$, the input token feature is first formed by concatenating positional information and functional features, followed by an MLP preprocessing layer.
	When \texttt{unified\_pos=True}, the coordinate input is replaced by a unified positional representation based on a reference grid of size $\texttt{ref}\times\texttt{ref}$.
	The resulting token features are passed through stacked Transolver blocks.
	Each block consists of layer normalization, a physics-attention module, a residual connection, and a feedforward MLP.
	The last block projects the hidden feature to the output field dimension.
	
	For most two-dimensional regular-grid Transolver baselines, we use
	\[
	n_{\mathrm{hidden}}=128,\qquad
	n_{\mathrm{layers}}=4,\qquad
	n_{\mathrm{head}}=8,\qquad
	\texttt{slice\_num}=64,\qquad
	\texttt{mlp\_ratio}=1,\qquad
	\texttt{dropout}=0.
	\]
	We use GELU activations and set the reference-grid parameter to $\texttt{ref}=8$.
	The structured-grid models use \texttt{unified\_pos=True} unless otherwise specified.
	
	\paragraph{(A) Regular-domain SISO benchmarks.}
	For boundary-to-grid tasks such as Laplace and source-free Poisson--Boltzmann equations, the input boundary vector
	$g\in\mathbb{R}^{N_b}$ is first mapped to grid-wise functional features through a global lifting MLP:
	\[
	g \mapsto F_g \in \mathbb{R}^{N^2\times d_f}.
	\]
	In our implementation, this lifting module is
	\[
	\texttt{MLP}(N_b \to 256 \to 256 \to N^2 d_f),
	\]
	with GELU activations and $d_f=1$ for these SISO tasks.
	The lifted feature $F_g$ is concatenated with positional features and passed to the structured 2D Transolver backbone, which predicts the grid solution
	$u\in\mathbb{R}^{N^2}$.
	The same boundary-lift strategy is used for the Laplace and source-free PB benchmarks, with task-dependent grid sizes.
	
	\paragraph{(B) Burgers equation.}
	For the Burgers benchmark, the input initial condition $u_0(x)$ is represented on the spatial grid and broadcast along the temporal direction to form a full space--time functional field on the $(x,t)$ grid.
	The resulting field, together with the space--time coordinates, is processed by the structured 2D Transolver backbone.
	The model outputs the solution trajectory $u(x,t)$ on the full $N_x\times N_t$ space--time grid.
	This treatment follows the same token-based structured-grid formulation as the steady two-dimensional tasks, but with the two token coordinates corresponding to space and time.
	
	\paragraph{(C) Darcy flow benchmarks.}
	For Darcy flow, the input coefficient field $a(x,y)$ is already defined on the target grid.
	Therefore, no boundary-to-grid lifting is required.
	The coefficient field is used as a grid-wise functional input channel and concatenated with positional information before entering the structured 2D Transolver backbone.
	We use this formulation for both continuous- and discontinuous-coefficient Darcy benchmarks, with grid sizes matching the corresponding datasets.
	
	\paragraph{(D) Regular-domain MISO benchmark: PB with source.}
	For the source-driven Poisson--Boltzmann benchmark, the input consists of a boundary trace $u_b\in\mathbb{R}^{N_b}$ and a source field $f\in\mathbb{R}^{N^2}$.
	The source field is already a full-grid input and is kept as one functional channel.
	The boundary trace is mapped to additional grid-wise channels through a global lifting MLP:
	\[
	u_b \mapsto F_{u_b}\in\mathbb{R}^{N^2\times d_b}.
	\]
	In our implementation, we use $d_b=4$ and a lifting hidden width of 256.
	The final functional input to Transolver is the concatenation
	\[
	F = [\, f,\; F_{u_b}\,]\in\mathbb{R}^{N^2\times (1+d_b)}.
	\]
	This concatenated grid-wise feature is then processed by the structured 2D Transolver backbone to predict the solution field $u$.
	
	\paragraph{(E) Regular-domain MIMO benchmark: PNP.}
	For the coupled PNP system, the input contains three boundary traces
	$(g_\phi,g_{c_+},g_{c_-})$, and the output contains three fields
	$(\phi,c_+,c_-)$.
	To construct a grid-compatible input for Transolver, we place the three boundary traces on the perimeter of a full grid and fill the interior values with zeros, producing three full-grid input channels.
	These channels are used as the functional input to the structured 2D Transolver backbone.
	The final projection layer is modified to output three channels simultaneously, corresponding to the three coupled fields.
	The training loss is the mean of the decoded physical-scale relative $\ell_2$ errors over the three output fields, consistent with the evaluation protocol.
	
	\paragraph{(F) Irregular-domain PB benchmarks.}
	For irregular finite-element PB benchmarks, we use an irregular-mesh Transolver variant.
	Each FEM node is treated as a token with coordinate $x_i\in\mathbb{R}^2$.
	The boundary vector $g\in\mathbb{R}^{N_b}$ is lifted to node-wise functional features through a global MLP:
	\[
	g \mapsto F_g \in \mathbb{R}^{N_p\times d_f},
	\]
	where $N_p$ is the number of FEM nodes.
	The node coordinates and lifted node-wise features are then passed to the irregular Transolver backbone to predict the nodal solution vector
	$u\in\mathbb{R}^{N_p}$.
	For these irregular-domain benchmarks, we use
	\[
	n_{\mathrm{hidden}}=128,\qquad
	n_{\mathrm{layers}}=8,\qquad
	n_{\mathrm{head}}=8,\qquad
	\texttt{slice\_num}=64,\qquad
	\texttt{ref}=8,\qquad
	\texttt{unified\_pos=False},
	\]
	with GELU activations, \texttt{mlp\_ratio}=1, and dropout set to zero.
	
	\paragraph{(G) 3D Poisson--Boltzmann benchmark.}
	For the 3D PB benchmark, we use a structured 3D Transolver variant.
	The unique boundary vector $g\in\mathbb{R}^{N_b}$ is first lifted to a full-grid functional field on the $N\times N\times N$ grid.
	The lifted feature is then combined with three-dimensional positional information and processed by a 3D structured Transolver backbone.
	Compared with the 2D structured-grid case, the 3D model uses a lighter attention configuration to control memory usage while keeping the same overall token-based design.
	The final output is the full volumetric solution field $u\in\mathbb{R}^{N^3}$.
	
	\paragraph{Training setup.}
	All Transolver baselines are trained using the same general evaluation protocol as the other methods: a fixed train/test split, train-only normalization, AdamW optimization, StepLR with $\gamma=1.0$, and relative $\ell_2$ errors computed after decoding predictions back to the physical scale.
	For multi-field PNP outputs, the reported error is the average relative $\ell_2$ error over the three physical fields.
	
\section{Visualization of predictions}
\label{app:vis}

In this section, we provide qualitative visualizations of the LNF-NO predictions on unseen test samples. 
Unless otherwise stated, the columns display:
(1) The \textbf{Reference Solution} (Ground Truth) obtained from the numerical solver;
(2) The \textbf{Prediction} output by our LNF-NO model;
(3) The pointwise \textbf{Absolute Error} map ($|u - \hat{u}|$).
For tasks with spatially varying coefficients or source terms, the input field is also visualized in the first column.

\textbf{Note:} Visualizations for the 3D Poisson--Boltzmann benchmark are presented in the main text (see the corresponding Figure~\ref{fig:pb3d_main}) and are therefore omitted here to avoid redundancy.

\begin{figure}[htbp]
	\centering
	\includegraphics[width=0.9\linewidth]{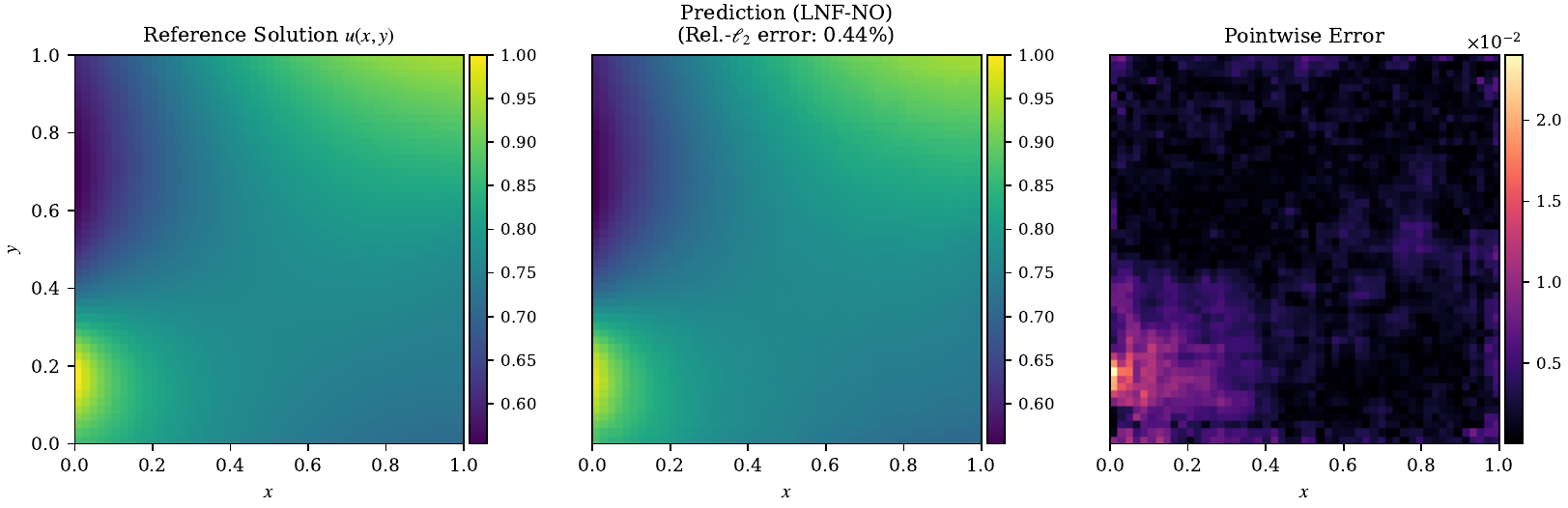}
	\caption*{\textbf{(a) Laplace Equation:} Prediction of the harmonic function on a regular grid.}
	\vspace{1em}
	\includegraphics[width=0.9\linewidth]{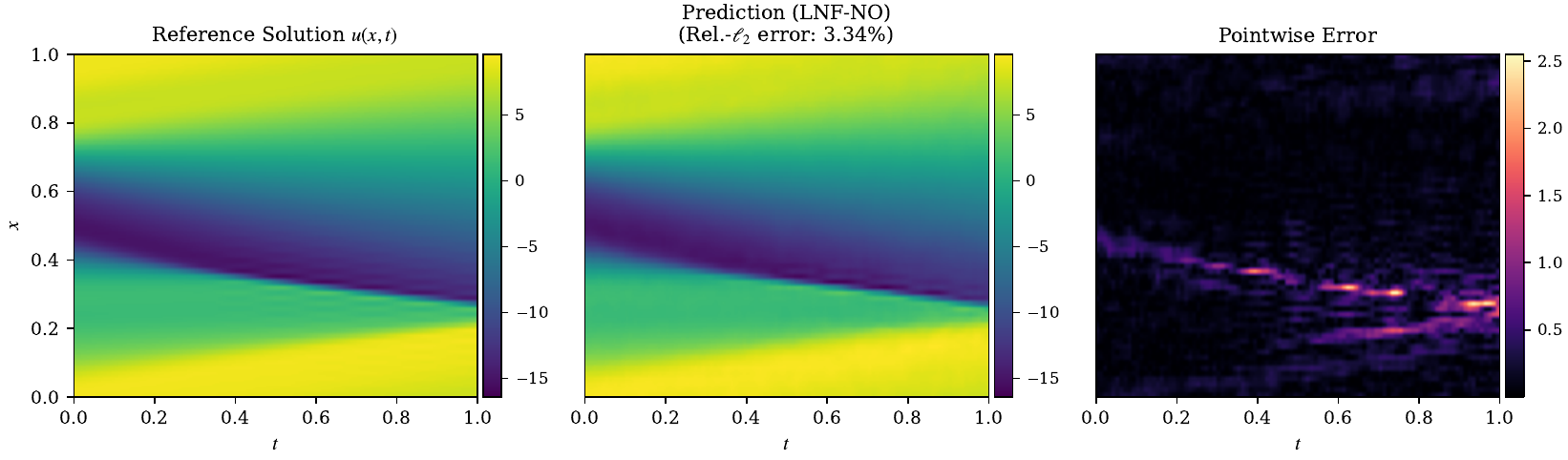}
	\caption*{\textbf{(b) Burgers Equation:} A snapshot of the wave propagation. The model sharply resolves the shock front with minimal oscillation.}
	
	\caption{\textbf{Standard Benchmarks.} Visualizations for (a) the steady-state 2D Laplace equation and (b) the time-dependent 1D Burgers equation.}
	\label{fig:vis_basic}
\end{figure}

\clearpage

\begin{figure}[htbp]
	\centering
	\includegraphics[width=0.9\linewidth]{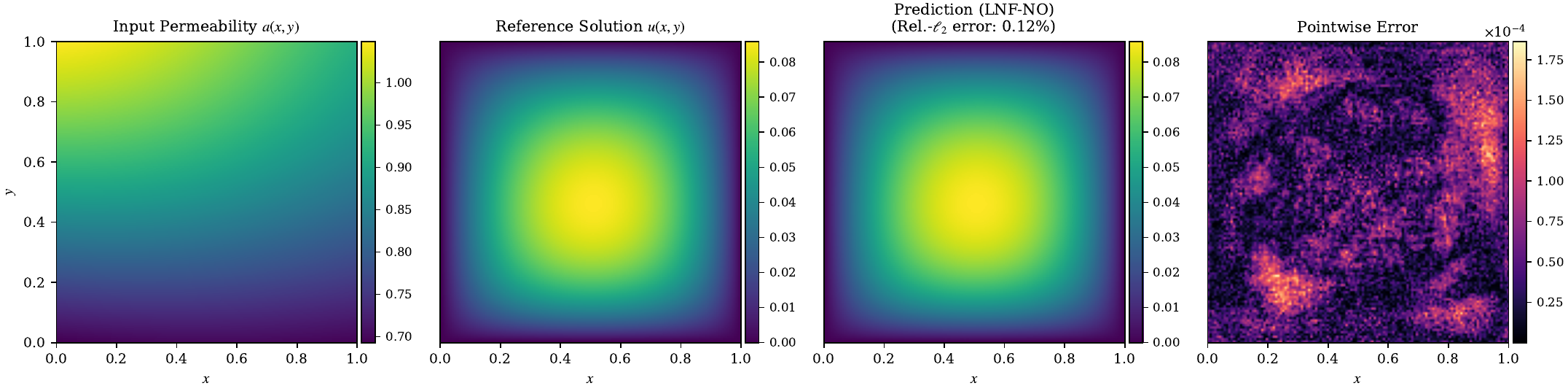}
	\caption{\textbf{Darcy Flow (Continuous Coefficients).} From left to right: Input Permeability $a(x,y)$, Reference Solution $u(x,y)$, Prediction, and Absolute Error.}
	\label{fig:vis_darcy_cont}
\end{figure}

\begin{figure}[htbp]
	\centering
	\includegraphics[width=0.9\linewidth]{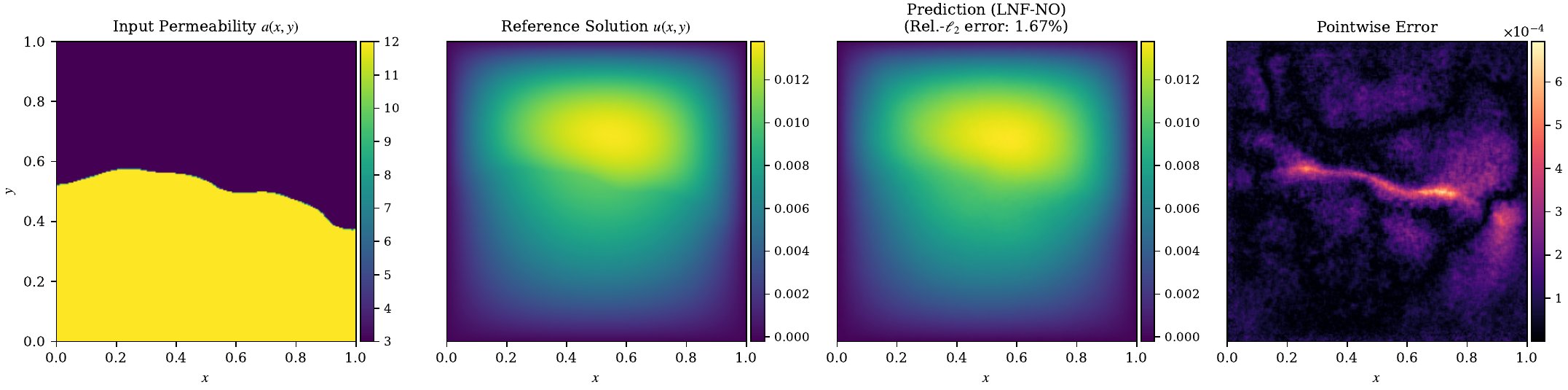}
	\caption{\textbf{Darcy Flow (Discontinuous Coefficients).} The model effectively handles sharp transitions in the solution field caused by jump coefficients.}
	\label{fig:vis_darcy_dis}
\end{figure}

\clearpage

\begin{figure}[htbp]
	\centering
	\includegraphics[width=0.9\linewidth]{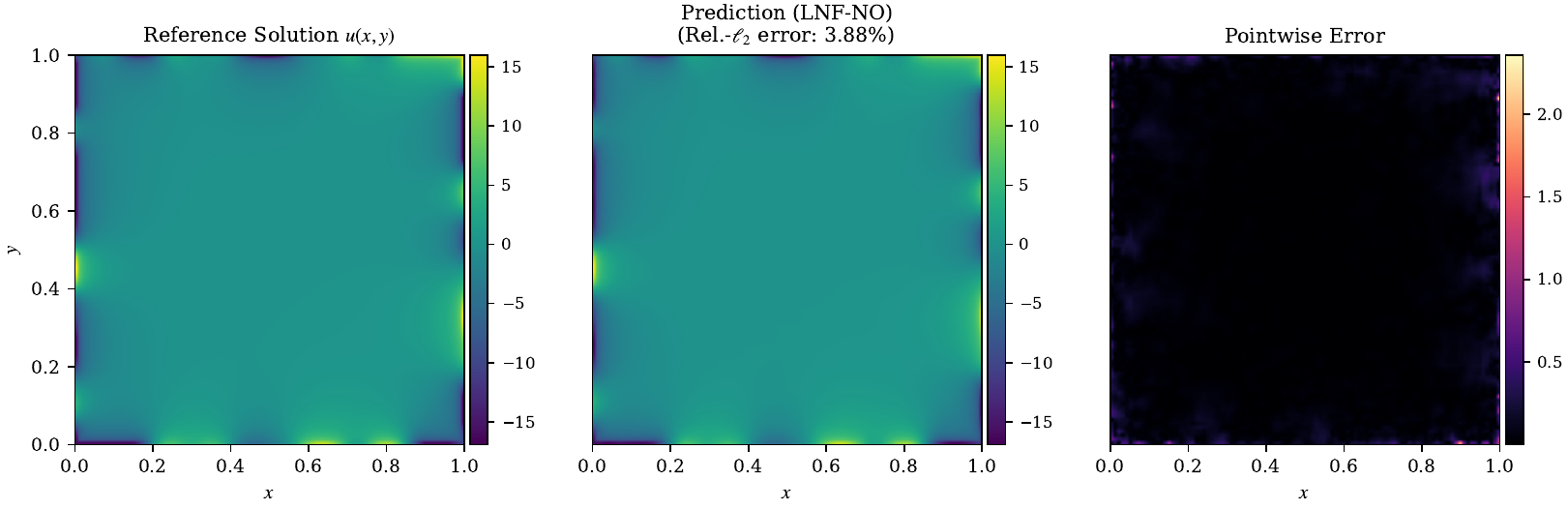}
	\caption{\textbf{Poisson--Boltzmann (High Nonlinearity $k=100$).} The model captures the sharp boundary layers induced by the strong sinh nonlinearity.}
	\label{fig:vis_pb_k100}
\end{figure}

\begin{figure}[htbp]
	\centering
	\includegraphics[width=0.9\linewidth]{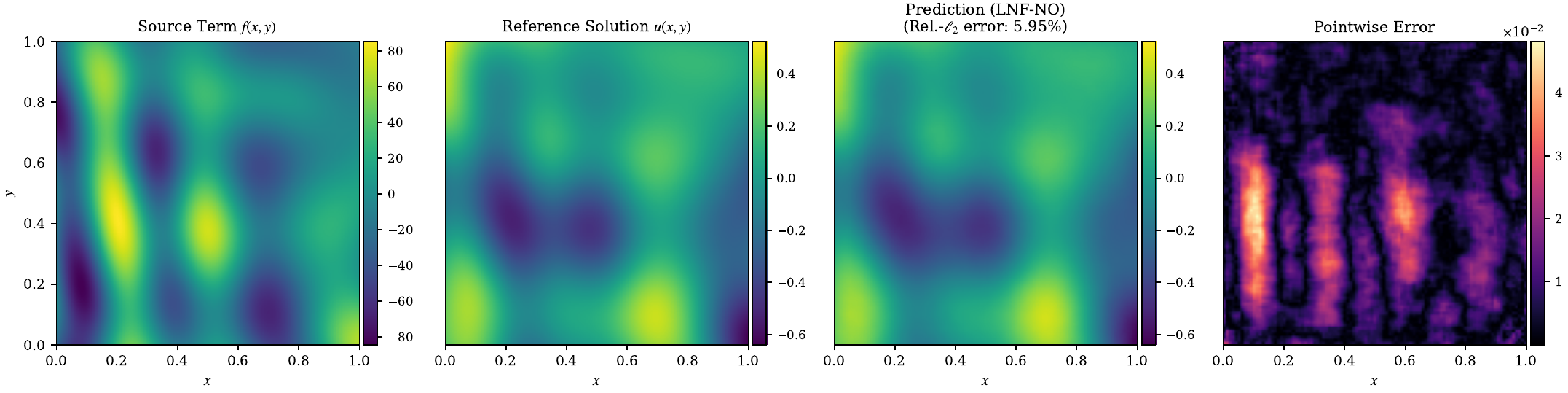}
	\caption{\textbf{Poisson--Boltzmann (Source-Driven).} Includes the variable source term $f(x,y)$ (first column) as input.}
	\label{fig:vis_pb_source}
\end{figure}

\clearpage

\begin{figure}[htbp]
	\centering
	\includegraphics[width=0.9\linewidth]{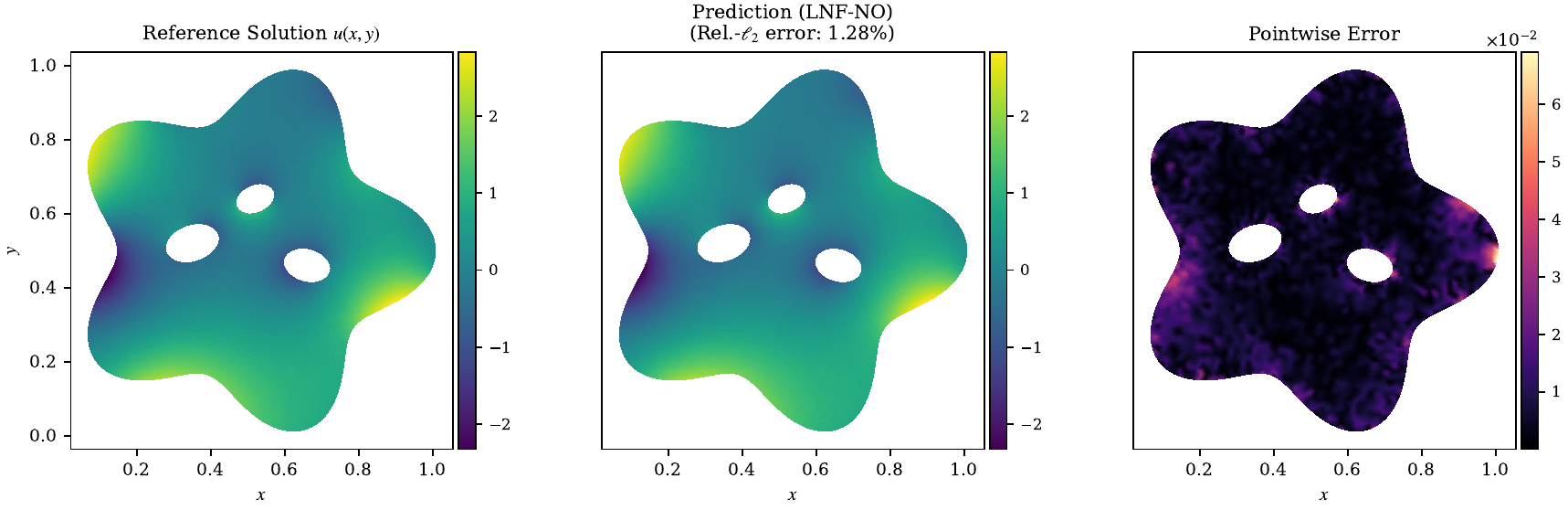}
	\caption{\textbf{Geometric Generalization (Irregular PB).} 
		Predictions on complex non-simply-connected domains, including the ``Flower'' shape with elliptical holes. 
		Since this task uses the node-based architecture, results are plotted directly on the finite-element mesh triangulation. 
		LNF-NO accurately captures the solution field along the intricate curved boundaries, confirming its capability to handle unstructured data and complex topologies without regular grid interpolation.}
	\label{fig:vis_irregular}
\end{figure}

\clearpage

\begin{figure}[htbp]
	\centering
	\includegraphics[width=0.9\linewidth]{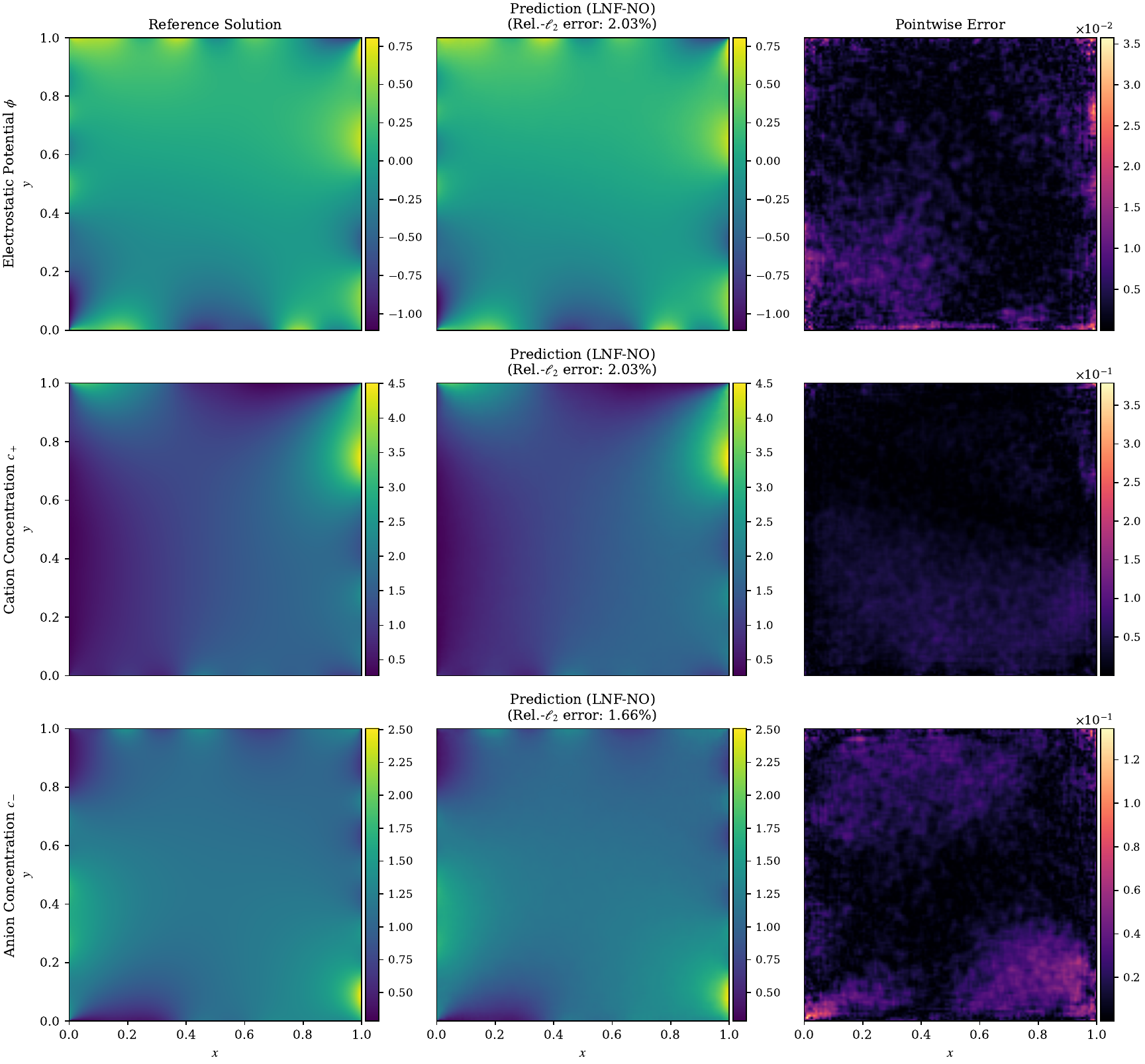}
	\caption{\textbf{Poisson--Nernst--Planck (PNP) System.} 
		Simultaneous prediction of three strongly coupled fields: Electrostatic Potential $\phi$ (top), Cation Concentration $c_+$ (middle), and Anion Concentration $c_-$ (bottom).
		The input boundary conditions (not shown) drive complex internal distributions. 
		The error maps (right column) demonstrate that LNF-NO maintains physical consistency across the variables and accurately resolves the sharp concentration gradients (Debye layers) near the boundaries, which are typical in electro-diffusion problems.}
	\label{fig:vis_pnp}
\end{figure}

\end{document}